\documentclass[12pt, reqno]{amsart}

\author[Michele Caprio, Mengqi Chen, Siu Lun Chau, Krikamol Muandet]{Michele Caprio, Mengqi Chen, Siu Lun Chau, Krikamol Muandet}
\address{University of Warwick}
\email{\{Michele.Caprio,Mengqi.Chen.2\}@warwick.ac.uk}
\address{Nanyang Technological University}
\email{siulun.chau@ntu.edu.sg}
\address{CISPA Helmholtz Center for Information Security}
\email{muandet@cispa.de}
\keywords{Fixed point theory; credal hyperspaces; order-preserving dynamics;
set-valued dynamics; imprecise Markov semigroups; Wasserstein contraction}
\subjclass[2020]{Primary 47H10, 54H25;
Secondary 47D07, 60J35, 28A33}
 
\title{Fixed Points and Envelope Dynamics on Credal Hyperspaces}

\usepackage[T1]{fontenc}
\usepackage{amsmath}
\usepackage{amssymb}
\usepackage{lscape}
\usepackage{centernot}
\usepackage{amsthm}
\usepackage[left=2.5cm, right=2.5cm, top=3cm]{geometry}
\usepackage[dvipsnames]{xcolor}  
\usepackage{hyperref,times,listings}

\AtBeginDocument{%
  \hypersetup{
    citecolor  = OliveGreen,
    linkcolor  = blue,
    urlcolor   = cyan
  }%
}
\usepackage{algorithm}
\usepackage{bbm}
\usepackage{booktabs}
\usepackage{algorithm}
\usepackage{algpseudocode}
\usepackage{tikz}
\usepackage{tikz-cd}
\usepackage{caption}
\usepackage{subcaption}
\usepackage{fancyhdr}
\usepackage[inline]{enumitem}
\usepackage{comment}
\usepackage{nicefrac}
\usepackage{bm}
\usepackage{mathrsfs}
\usepackage{multirow}
\usepackage{graphicx}
\usepackage[utf8]{inputenc}
\usepackage{cancel}
\usepackage{mathtools}
\usepackage{natbib}

\newcommand{\vertiii}[1]{{\left\vert\kern-0.25ex\left\vert\kern-0.25ex\left\vert #1 
    \right\vert\kern-0.25ex\right\vert\kern-0.25ex\right\vert}}

\AtBeginDocument{%
   \def\MR#1{}
}

\makeatletter
\def\algbackskip{\hskip-\ALG@thistlm}
\makeatother

\makeatletter
\def\namedlabel#1#2{\begingroup
    #2%
    \def\@currentlabel{#2}%
    \phantomsection\label{#1}\endgroup
}
\makeatother

\newtheorem{thm}{Theorem}[section]

\theoremstyle{definition} 
\newtheorem{defi}[thm]{Definition}
\let\olddefi\defi
\renewcommand{\defi}{\olddefi\normalfont}
\newtheorem{rmk}{Remark}
\let\oldrmk\rmk
\renewcommand{\rmk}{\oldrmk\normalfont}

\newcommand{\Prob}{\mathcal{P}}
\newcommand{\Cred}{\mathscr{C}}

\newcommand{\co}{\operatorname{co}}
\newcommand{\clco}{\overline{\operatorname{co}}}
\newcommand{\dist}{\operatorname{dist}}
\newcommand{\Fix}{\operatorname{Fix}}
\newcommand{\Borel}{\mathcal{B}}

\newcommand{\Lip}{\operatorname{Lip}}
\newcommand{\Id}{\operatorname{Id}}

\newtheorem{theorem}{Theorem}
\newtheorem{lemma}[theorem]{Lemma}
\newtheorem{proposition}[theorem]{Proposition}
\newtheorem{corollary}{Corollary}[theorem]

\newtheorem{definition}[theorem]{Definition}
\newtheorem{remark}{Remark}
\newtheorem{example}{Example}

\hypersetup{
    pdftitle={prove},
    pdfauthor={autori},
    pdfmenubar=false,
    pdffitwindow=true,
    pdfstartview=FitH,
    colorlinks=true,
    linkcolor=blue,
    citecolor=green,
    urlcolor=cyan
}

\providecommand{\MR}[1]{}

\providecommand{\MR}{\relax\ifhmode\unskip\space\fi MR }

\makeatletter
\newcommand{\skipitems}[1]{%
  \addtocounter{\@enumctr}{#1}%
}
\makeatother

\begin{document}

\begin{abstract}
Let $X$ be a compact metric space and consider the hyperspace
$\mathscr C_X$ of nonempty weakly closed convex sets of Borel probability
measures on $X$, called credal sets. These sets represent uncertainty about
which probability law governs the phenomenon of interest. We equip $\mathscr C_X$
with the Hausdorff metric induced by a metric for weak convergence.
Credal order intervals
$[\mathcal A,\mathcal B]
=
\{
\mathcal P\in\mathscr C_X:
\mathcal A\subseteq\mathcal P\subseteq\mathcal B
\}$
are complete lattices, and every inclusion-monotone, Hausdorff-continuous
self map of such an interval has least and greatest fixed points obtained as
limits of iteration from $\mathcal A$ and $\mathcal B$. 
For families of inclusion-monotone updates that may vary with time,
we construct lower and upper envelope maps that bound every admissible
trajectory. When these envelope maps are Hausdorff-continuous, the distance
of every trajectory from a limiting credal interval vanishes uniformly over
all switching sequences and initial credal sets.
Coalescence of the limiting interval endpoints yields convergence to a
common credal limit. We then construct the credal evolution induced by a
time consistent Nisio semigroup. Uniform $W_1$ contraction of the underlying
Markov semigroups induces Hausdorff-Wasserstein contraction of the credal
evolution, yielding a unique common invariant credal set, exponential
convergence from every initial credal set, and a support functional equal to
the robust ergodic limit. 
We also derive perturbation bounds from finite time Nisio discrepancies and
one step kernel set discrepancies, and characterize the countable sequences
of support coordinates that represent credal sets, providing an explicit
reconstruction formula.
\end{abstract}

\maketitle
\thispagestyle{empty}




\section{Introduction}\label{sec:introduction}
A credal set is a nonempty, weakly closed, and convex collection of
probability measures used to represent partial information about an
underlying probability law; it is a workhorse in imprecise probability
theory
\citep{levi1980enterprise,walley1991statistical,
augustin2014introduction,troffaesDeCooman2014lowerPrevisions}.
A singleton credal set recovers the usual precise probabilistic model, while
a larger set represents ambiguity about the underlying distribution.

Let $X$ denote our state space of interest, and let $\Cred_X$ denote the
collection of credal sets on $X$. An {\em update} is a rule that maps a
current credal set to a new one, for example by propagating its probability
measures through a transition mechanism, incorporating new information, or forming the closed convex hull
of a prescribed collection of probability measures. We write
$\mathcal P_n\in\Cred_X$ for the credal set after $n$ update steps.

If the same map $F:\Cred_X\to\Cred_X$ is applied at every step, the dynamics
are {\em autonomous},
$\mathcal P_{n+1}=F(\mathcal P_n)$.
Here autonomous means that the update rule does not depend explicitly on
$n$. The fixed map $F$ may nevertheless have been constructed from data,
evidence, or a prescribed transition mechanism. For such a map, a fixed credal set represents an uncertainty model that
remains unchanged under the update. Convergence toward a fixed credal set
describes stabilization of the entire collection of plausible probability
measures, rather than of a single selected probability law.

If the update rule selected at time $n$ depends on incoming information or
on an adaptive choice $e_n$, the dynamics are {\em nonautonomous},
$\mathcal P_{n+1}=F_{e_n}(\mathcal P_n)$,
where $e_n$ indexes the information or choice at time $n$, and
$F_{e_n}:\Cred_X\to\Cred_X$ is the corresponding update map. 

The resulting fixed point problem has several logically distinct levels.
One may ask whether a fixed credal set exists, whether inclusion-monotone
updates possess least and greatest fixed credal sets, whether those extremal
points are reached by ordinary iteration, and whether trajectories remain
controlled when the update rule changes with time. These questions are
answered by different combinations of topological, order, and metric
structure: compactness and convexity support existence results, set inclusion
yields extremal fixed points and comparison dynamics, and metric contraction
gives uniqueness and convergence. The purpose of this paper is to identify
these structures on credal hyperspaces and to show how they interact.

These questions are not settled merely by applying a familiar fixed point
theorem. The relevant topological and order structures must first be
established for a space whose elements are themselves sets of probability
measures. Moreover, Knaster-Tarski guarantees extremal fixed points on
complete lattices but does not, by itself, imply convergence of ordinary
endpoint iteration \citep{tarski1955,cousot1979constructive}. When the
update rule changes at each step, the difficulty is different: the resulting
trajectory is not generated by repeated application of a single map, so
fixed point existence for the individual updates does not directly
control its asymptotic behavior.

The imprecise Markov setting presents a further question. The time
consistent nonlinear evolution of an imprecise Markov semigroup is
initially formulated in terms of expectations of bounded observables
\citep{imsg}. To obtain a corresponding fixed point theory for credal
sets, one must identify an evolution acting on sets of probability
measures and determine whether contraction properties of the underlying
Markov models lead to an invariant credal set and quantitative stability.

Throughout most of the paper, we assume that the underlying state space
$X$ is compact metric. This is a sufficient rather than a necessary
condition: for instance, the order theoretic arguments only require
compactness of the prescribed upper credal bound, while extending the
Wasserstein analysis to noncompact Polish spaces would require suitable
invariant domains and control of moments and tightness. We work in the
compact setting to avoid these additional technical complications and
focus on the fixed point and dynamical questions that are central to the
paper.

\textbf{Main Contributions.}
The main results of the paper, together with the sections where they are
established, are as follows:
\begin{enumerate}[label=(\roman*),leftmargin=*]

\item We characterize which sequences of real numbers can arise as support
values of credal sets evaluated on a countable dense family of continuous
functions, and provide an explicit formula to recover the unique credal
set represented by each such sequence. We also establish the compactness
and completeness properties needed for the fixed point analysis and derive
credal set specializations of classical fixed point theorems. The credal hyperspace
and support coordinate results are developed in
Section~\ref{sec:preliminaries}, while the classical fixed point
specializations are given in Section~\ref{sec:classical}.

\item We prove that every credal order interval is a complete lattice. For an
inclusion-monotone and Hausdorff-continuous self-map, iteration from the lower
and upper endpoints converges in Hausdorff distance to the least and greatest
fixed points. We also identify the least fixed point of a natural hull map as
the smallest closed convex forward invariant uncertainty class containing a
prescribed initial credal set. These results are developed in Section~\ref{sec:order}.

\item For families of inclusion-monotone updates that may vary arbitrarily
with time, we construct lower and upper envelope dynamics that bound every
admissible trajectory at each finite step. The endpoint envelopes converge
to limiting fixed credal sets, and every admissible orbit satisfies a
quantitative distance to interval estimate that tends to zero uniformly over
all switching sequences and all initial credal sets in the prescribed
interval. Consequently, all Hausdorff cluster points lie in the limiting
interval, while coalescence of its endpoints implies convergence to a common
path independent limit. These results are developed in Section~\ref{sec:varying}.

\item For imprecise Markov semigroups satisfying a uniform Wasserstein
contraction, we construct the induced time consistent credal semigroup and
prove that it is contractive in Hausdorff-Wasserstein distance. It has a
unique common invariant credal set, every initial credal set converges to it
exponentially, and its support functional coincides with the robust ergodic
limit of the Nisio semigroup. We further prove quantitative perturbation
stability: the finite time discrepancy between two Nisio semigroups is
exactly the uniform discrepancy between their induced credal maps, while
one step kernel set estimates yield an explicit bound on the distance
between their invariant credal sets. These results are developed in Section~\ref{sec:imsg}.

\end{enumerate}
A short Bayesian tilting subsection serves as a boundary example: it verifies
Hausdorff continuity and the exact finite generator formula while showing
that positivity alone supplies no strict projective contraction. This example is presented in Section~\ref{sec:bayes}. The remainder of the paper contains the related work and technical developments described above, followed by concluding remarks in Section~\ref{sec:conclusion}.


\section{Related Work}\label{sec:related}

This section discusses related work on credal sets and updating, Bayesian
credal models, imprecise stochastic processes and semigroups, convex
hyperspaces, classical fixed point theory, and order preserving dynamics. These strands enter the paper at two different levels. The probabilistic
literature motivates the models and update mechanisms we study, while the
work on convex hyperspaces, fixed point theory, and ordered dynamics provides
the mathematical structure used to analyze their fixed points and long term
behavior.

\textbf{Credal Sets, Lower Previsions, and Updating.}
{\em Credal sets} are one of the standard representations of imprecise probabilistic
information \citep{levi1980enterprise,walley1991statistical,
augustin2014introduction}.\footnote{Closely related formalisms include coherent lower
previsions and sets of desirable gambles 
\citep{miranda2008survey,troffaesDeCooman2014lowerPrevisions}.} In the present paper, we work with nonempty weakly closed convex sets
of \emph{countably additive Borel probability measures}. This is narrower
than Walley's general lower prevision framework, which does not impose
countable additivity in the same way. The restriction is useful for the
topological framework developed below. On a compact metric state space, the
countably additive Borel probability measures form a compact metrizable
space under weak convergence, and the Prokhorov metric metrizes this
convergence. Consequently, weakly closed credal sets are compact, allowing
us to work with their Hausdorff hyperspace and to apply compactness based
fixed point arguments. Convexity, in turn, ensures stability under
probabilistic mixtures and permits the support functional representation
used throughout the paper.

There is no single operation that should be called ``the'' update of an
imprecise probability model. Revision and conditioning rules for convex sets
of probabilities and coherent previsions have been studied from several
perspectives \citep{moralWilson1995revision,miranda2009updating,
walley1991statistical}. Different update rules may be defined on different classes of credal sets and
may preserve coherence under different conditions.\footnote{Here coherence means that
the resulting lower previsions respect sure gains, positive scaling, and
superadditivity, so that their assessments cannot be combined to produce a
guaranteed loss.} Consequently, these rules cannot always be treated as self
maps of a single credal domain.
 Continuity
is itself a substantive issue: \citet{cattaneo2014continuous} studies
continuous updating precisely from a robustness perspective, while
\citet{gong2021judicious} illustrate how alternative conditioning rules can
display qualitatively different phenomena such as dilation and sure loss.
Accordingly, our fixed point results are formulated for an abstract credal set
self-map whose topological or order-theoretic properties must be verified
separately. 

\textbf{Bayesian Credal Models.}
Credal sets have been used in machine learning to separate different sources
of predictive uncertainty \citep{hullermeierWaegeman2021}. A credal set is
called {\em finitely generated} when it is the convex hull of finitely many
probability measures. Credal Bayesian Deep Learning propagates a
finite family of priors and likelihood models through Bayesian tilting
\citep{caprio2024credal}. This construction motivates the concrete
operator studied in Subsection~\ref{sec:bayes}.

\textbf{Markov Set Chains, Imprecise Markov Chains, and Semigroups.}
The closest asymptotic literature to the present paper concerns Markov
dynamics under imprecisely specified transitions. An early set-valued
approach is the theory of Markov set chains developed by
\citet{hartfielSeneta1994markovSetChains} and subsequently treated
systematically by \citet{hartfiel1998markovSetChains}. There, the transition
mechanism is described through sets of stochastic matrices, and questions of
ergodicity, regularity, invariant behavior, and convergence are studied by
exploiting matrix and Markov structure.

Imprecise Markov chains provide a related but distinct formulation based on
credal transition models or, equivalently in the finite state setting, lower
and upper expectations, obtained by taking the infimum and supremum of the
expected value over the admissible transition probabilities. \citet{deCooman2009ImpreciseMCsLimit} study the
evolution of the inferred credal model and give conditions under which it
converges to a uniquely invariant credal set independently of the initial
credal set. \citet{skulj2009intervalMarkov} develops finite discrete time
chains with interval valued transition probabilities and studies regularity
and convergence. At the operator level,
\citet{hermansDeCooman2012ergodic} characterize ergodicity of upper
transition operators and obtain Perron-Frobenius type convergence criteria.
Coefficients of ergodicity and the associated convergence rates are
developed by \citet{skuljHable2013coefficients}, while
\citet{skulj2013invariantDistributions} analyzes invariant distributions and
convergence through accessibility and permanent class structure.

The stochastic process theory extends beyond convergence of one-step
marginal credal models. \citet{deCoomanDeBockLopatatzidis2016} develop global
models for discrete time imprecise stochastic processes and prove pointwise
ergodic theorems for imprecise Markov chains.
\citet{tjoensDeBock2021weakErgodicity} study weaker forms of ergodicity
through convergence of lower and upper expected time averages. Successive
imprecise probability updates have also been studied using ergodic theoretic
tools outside the Markov chain framework. In particular,
\citet{caprioMukherjee2023ergodic} analyze repeated dynamic imprecise
probability kinematics on at most countable state spaces.

Imprecise Markov semigroups extend this discussion to continuous time and
to state spaces whose geometry can be used in the ergodic analysis.
\citet{imsg} define an imprecise Markov semigroup as a
nonempty family of Markov semigroups on a common state space, representing
ambiguity about both the transition mechanism and the invariant law. Their
framework uses geometric and probabilistic conditions, including curvature
estimates, gradient bounds, contractive couplings, and Wasserstein
contraction, to obtain long term bounds that are uniform in the initial
state. They also construct time consistent nonlinear evolutions by allowing the
selected Markov model to switch over time. Under a uniform Wasserstein contraction, the upper expected value of each
bounded Lipschitz observable converges to a limit independent of the initial
state, with a corresponding result for lower expected values. In this work,
we use this time consistent nonlinear evolution to define an evolution on
credal sets and prove that it has a unique invariant credal set attracting
every initial credal set exponentially in Hausdorff-Wasserstein distance.
For each bounded Lipschitz observable, the limiting upper expectation is the
supremum of its ordinary expected values over this invariant credal set.

\textbf{Convex Hyperspaces.}
The compact convex hyperspace structure we use
belongs to a classical line of convex analysis. R{\aa}dstr\"om's embedding
theorem provides an additive and positively homogeneous linearization of
spaces of compact convex sets \citep{Radstrom1952}; H\"ormander develops the
support functional representation in locally convex spaces
\citep{hormander1955support}. Countable functional coordinate embeddings in
metrizable locally convex settings are studied by
\citet{nadler1979hyperspaces}, and standard Hausdorff hyperspace facts are
treated systematically by \citet{beer1993topologies}. Our use of these ideas specializes this classical machinery to compact convex
sets of probability measures.


\textbf{Order-Preserving Dynamics and Lattice Fixed Points.}
The order-theoretic part of the paper builds on the Knaster-Tarski theorem
\citep{tarski1955} and its constructive variants
\citep{cousot1979constructive}. More broadly, stability and convergence of
discrete order-preserving dynamical systems have been studied under convexity
and monotonicity assumptions 
\citep{akianGaubertLemmens2011}. Our setting is more specific: the ordered
space is the interval
  $[\mathcal A,\mathcal B]
  =\{\mathcal P:\mathcal A\subseteq\mathcal P\subseteq\mathcal B\}$
inside a compact credal hyperspace. The common nonempty lower bound $\mathcal A$ ensures that arbitrary
intersections remain nonempty. Thus intersections give greatest lower bounds,
while closed convex hulls of unions give least upper bounds.

After the first version of this manuscript appeared,
\citet{heitzig2026oracles} independently used Knaster-Tarski in a
self consistent set-valued fixed point problem and considered a closed convex
restriction together with an affine orbit hull construction. The common
lattice-theoretic ingredient is classical. The present paper differs in both
the ordered domain and the dynamical conclusions: the domain is a nonempty
credal interval with a prescribed lower bound, and compactness plus Hausdorff
continuity makes the ordinary countable endpoint iterates converge to the
extremal fixed credal sets. Knaster-Tarski alone yields the existence of least and greatest fixed points,
but not the Hausdorff convergence of the ordinary endpoint iterates
that we establish here.


\section{Credal Hyperspaces: Compactness and Support Coordinates}\label{sec:preliminaries}
This section develops the topological and convex analytic structure of the credal hyperspace used throughout the paper. We begin by collecting several standard facts about weak convergence, Hausdorff hyperspaces, and closed convex hulls, specialized to the present setting. Then, we present support coordinate results, which constitute the main contribution of this section. Lemma~\ref{lem:affine-embedding} represents each credal set by a countable sequence of real valued support coordinates, while Theorem \ref{thm:support-coordinate-characterization} characterizes exactly which sequences can arise in this way and shows how the credal set can be recovered uniquely from such a sequence. Finally, Corollary~\ref{cor:possibility-nested-cuts} relates Theorem~\ref{thm:support-coordinate-characterization} to recent developments in imprecise probabilistic machine learning.

Throughout, probability measures are countably additive Borel probability measures. Let $X$ be a (nonempty) metric space, let $\Borel(X)$ be its Borel $\sigma$-algebra, write $\Prob(X)$ for the set of probability measures on $(X,\Borel(X))$, and write $C(X)$ for the space of real valued continuous functions on $X$. For $P\in\Prob(X)$ and a bounded measurable function $f$, use the shorthand $Pf\coloneqq\int_X f \mathrm dP$. Weak convergence is denoted by $P_n\Rightarrow P$.

\begin{definition}[Credal Set and Credal Hyperspace]
A \emph{countably additive credal set} on $X$ is a nonempty weakly closed convex subset of $\Prob(X)$. For the metric fixed point analysis below, we write
\[
  \Cred_X
  \coloneqq\{\mathcal P\subseteq\Prob(X):\mathcal P\neq\emptyset,
       \mathcal P\text{ is weakly compact and convex}\}.
\]
Thus $\Cred_X$ is the {\em compact credal hyperspace}. When $X$ is compact metric, $\Prob(X)$ is weakly compact and weak closedness is equivalent to weak compactness for subsets of $\Prob(X)$, so this notation includes all weakly closed convex credal sets on $X$.
\end{definition}

For compact metric $X$, the weak topology on $\Prob(X)$ can be metrized in a way that is convenient for convex combinations. Since $C(X)$ is separable, choose a sequence $(u_j)_{j\ge1}$ dense in the closed unit ball of $C(X)$ under the supremum norm and set
\begin{equation}\label{eq:rho}
 \rho(P,Q)
 \coloneqq\sum_{j=1}^{\infty}2^{-j-1}
   \left|\int_Xu_j \mathrm dP-\int_Xu_j \mathrm dQ\right|.
\end{equation}
The series is bounded by $1$.

\begin{lemma}[A Convex Metric for Weak Convergence]\label{lem:weak-metric}
If $X$ is compact metric, $\rho$ in \eqref{eq:rho} is a metric that metrizes the weak convergence on $\Prob(X)$. Moreover, for $P,Q,P',Q'\in\Prob(X)$ and $\lambda\in[0,1]$,
\begin{equation}\label{eq:convex-metric}
 \rho \left(\lambda P+(1-\lambda)Q,\lambda P'+(1-\lambda)Q'\right)
 \le \lambda\rho(P,P')+(1-\lambda)\rho(Q,Q').
\end{equation}
\end{lemma}

\begin{proof}
Weak convergence implies convergence of every coordinate in \eqref{eq:rho}; boundedness of the summands and a finite tail argument, then, give convergence in $\rho$. Conversely, if $\rho(P_n,P)\to0$, all coordinate integrals converge. Given $f\in C(X)$, set $M=\max\{1,\|f\|_\infty\}$ and choose $u_j$ with $\|f/M-u_j\|_\infty<\varepsilon$. Then
\[
 \left|P_n f-Pf\right|
 \le 2M\varepsilon+M\left|P_nu_j-Pu_j\right|,
\]
which proves $P_n\Rightarrow P$. Separation of probability measures by $C(X)$ shows that $\rho$ is a metric. Inequality \eqref{eq:convex-metric} follows coordinatewise from the triangle inequality and then by summation.
\end{proof}

Let $d_H$ denote the Hausdorff metric induced by $\rho$ on nonempty compact subsets of $\Prob(X)$,
\begin{equation}\label{eq:hausdorff}
 d_H(A,B)
 \coloneqq\max\left\{
   \sup_{P\in A}\dist_\rho(P,B),
   \sup_{Q\in B}\dist_\rho(Q,A)
 \right\}.
\end{equation}
Thus $d_H(A,B)$ measures how well the two sets approximate one another in the
metric $\rho$. A small value of $d_H(A,B)$ means that every probability
measure in $A$ is close, in $\rho$, to some probability measure in $B$, and
vice versa. Equivalently, $d_H(A,B)\leq\varepsilon$ if and only if every
element of either set lies within $\varepsilon$ of the other set. In
particular, when $A=\{P\}$ and $B=\{Q\}$ are singletons,
$d_H(A,B)=\rho(P,Q)$.
Since $\rho$ metrizes weak convergence on $\Prob(X)$, Hausdorff convergence
of credal sets means uniform mutual approximation of their elements with respect to weak convergence.
Then, we have the following.
\begin{lemma}[Credal Hyperspace Structure]\label{lem:hyperspace}
Let $X$ be compact metric. The following are true.
\begin{enumerate}[label=(\alph*),leftmargin=*]
\item $(\Prob(X),\rho)$ is compact.
\item $(\Cred_X,d_H)$ is compact.
\item If $A_n,B_n\in\Cred_X$, $A_n\to A$, $B_n\to B$ in $d_H$, and $A_n\subseteq B_n$ for all $n$, then $A\subseteq B$.
\end{enumerate}
If $(X,d)$ is a complete separable metric space and $\rho_P$ is the
Prokhorov metric on $\Prob(X)$ induced by $d$, then $\Prob(X)$ is complete,
and the hyperspace of nonempty compact subsets of $\Prob(X)$ is complete
in the Hausdorff metric induced by $\rho_P$; its subspace of compact convex
sets is closed and hence complete.
\end{lemma}

\begin{proof}
Part (a) follows from Prokhorov's theorem for compact $X$ \citep{billingsley1999}. The hyperspace of nonempty compact subsets of a compact metric space is compact in the Hausdorff metric \citep[Chapter~3]{beer1993topologies}. It remains to check that convexity is closed under Hausdorff limits. Suppose $\mathcal P_n\in\Cred_X$ and that
$\mathcal P_n\to\mathcal P$ in $d_H$, where $\mathcal P$ is a nonempty
compact subset of $\Prob(X)$. For $P,Q\in\mathcal P$, choose $P_n,Q_n\in\mathcal P_n$ with $P_n\to P$ and $Q_n\to Q$. By \eqref{eq:convex-metric},
\[
 \lambda P_n+(1-\lambda)Q_n\to\lambda P+(1-\lambda)Q.
\]
The left hand side belongs to $\mathcal P_n$, and Hausdorff convergence implies that its distance to $\mathcal P$ tends to zero. Since $\mathcal P$ is closed, the limit belongs to $\mathcal P$. Thus the compact set hyperspace has the closed subspace $\Cred_X$, proving (b). Part (c) follows similarly: for $P\in A$, choose $P_n\in A_n$ with $P_n\to P$. Since $P_n\in B_n$ and $B_n\to B$, we have $\dist_\rho(P_n,B)\le d_H(B_n,B)\to0$. Together with $P_n\to P$ and closedness of $B$, this gives $P\in B$.

For complete separable metric $X$, the Prokhorov metric on $\Prob(X)$ is
complete \citep{billingsley1999}. The Hausdorff hyperspace of nonempty compact subsets of a complete metric space is complete \citep[Chapter~3]{beer1993topologies}. The same convexity limit argument as above, using continuity of the mixture map in the weak topology, shows that the compact convex subspace is closed.
\end{proof}

The lemma separates the two metric settings used below. When $X$ is compact,
the entire credal hyperspace is compact, and inclusion is preserved under
Hausdorff limits. When $X$ is only Polish, global compactness may fail, but
the space of nonempty compact convex subsets remains complete. Accordingly,
the results below invoke compactness or completeness depending on the fixed
point theorem being applied.

For $A\subseteq\Prob(X)$, write $\clco(A)$ for its weakly closed convex hull. The next elementary fact will be used repeatedly.

\begin{lemma}[Closed Convex Hull is Hausdorff Nonexpansive]\label{lem:clco}
Let $X$ be compact metric and equip $\Prob(X)$ with $\rho$ from \eqref{eq:rho}. For nonempty compact $A,B\subseteq\Prob(X)$,
\[
 d_H\bigl(\clco(A),\clco(B)\bigr)\le d_H(A,B).
\]
\end{lemma}

\begin{proof}
Let $r>d_H(A,B)$. If
$P=\sum_{i=1}^m\lambda_iP_i\in\co(A)$, then for each
$i=1,\dots,m$ choose $Q_i\in B$ such that 
$\rho(P_i,Q_i)<r$. By \eqref{eq:convex-metric}, we have that
\[
 \rho\left(P,\sum_{i=1}^m\lambda_iQ_i\right)<r.
\]
Hence every point of $\co(A)$ lies within $r$ of $\co(B)$. Passing to closures and reversing the roles of $A$ and $B$ proves the result after letting $r\downarrow d_H(A,B)$.
\end{proof}

The final structural fact is classical. R\aa dstr\"om linearized spaces of compact convex sets, while H\"ormander gave a support functional representation in locally convex spaces \citep{Radstrom1952,hormander1955support}; \citet{nadler1979hyperspaces} used countably many continuous linear functionals in the metrizable locally convex setting. We record the specialization needed here and make its countable support coordinates explicit.

For $\mathcal P\in\Cred_X$ and $u\in C(X)$, define the support value
\begin{align}\label{supp-value}
    h_{\mathcal P}(u)\coloneqq\sup_{P\in\mathcal P}\int_Xu \mathrm dP.
\end{align}
For credal sets $\mathcal P,\mathcal Q$ and $\lambda\in[0,1]$, define their Minkowski convex combination by
\[
 \lambda\mathcal P+(1-\lambda)\mathcal Q
 \coloneqq\{\lambda P+(1-\lambda)Q:P\in\mathcal P, Q\in\mathcal Q\}.
\]
The mixture map is continuous, so this set is compact; convexity is immediate. Hence the Minkowski convex combination is again an element of $\Cred_X$.

\begin{lemma}[Classical Support Coordinate Embedding]\label{lem:affine-embedding}
If $X$ is compact metric, there is an affine homeomorphic embedding
\[
 J:\Cred_X\longrightarrow\mathbb R^{\mathbb N}
\]
into the sequence space with its product topology, whose image is compact and convex.
\end{lemma}

\begin{proof}
Let $(u_j)_{j\ge1}$ be the dense sequence used in \eqref{eq:rho} and set
\[
 J(\mathcal P)\coloneqq\bigl(h_{\mathcal P}(u_j)\bigr)_{j\ge1}.
\]
This is a countable support coordinate specialization of the classical support functional viewpoint of \citet{hormander1955support}; compare also the countable functional coordinate construction of \citet{nadler1979hyperspaces}. R\aa dstr\"om's theorem gives a related additive, positively homogeneous linearization for compact convex subsets of normed spaces \citep{Radstrom1952}.

Each coordinate $\mathcal P\mapsto h_{\mathcal P}(u_j)$ is continuous for Hausdorff convergence: the evaluation $P\mapsto Pu_j$ is uniformly continuous on compact $\Prob(X)$, so Hausdorff-close compact sets have close suprema. The map $J$ is injective. Indeed, if $\mathcal P\neq\mathcal Q$, say $P_0\in\mathcal P\setminus\mathcal Q$, the Hahn-Banach separation theorem applied to the weak$^\star$ locally convex space $C(X)^\star$ \citep{aliprantis} gives a weak$^\star$-continuous linear functional, hence some $u\in C(X)$, with
\[
 P_0u>\sup_{Q\in\mathcal Q}Qu.
\]
Density of $(u_j)$ in the unit ball, after rescaling $u$, preserves a strict separation for some coordinate $u_j$; hence $J(\mathcal P)\neq J(\mathcal Q)$.

Finally,
\[
 h_{\lambda\mathcal P+(1-\lambda)\mathcal Q}(u_j)
 =\lambda h_{\mathcal P}(u_j)+(1-\lambda)h_{\mathcal Q}(u_j),
\]
so $J$ is affine and $J(\Cred_X)$ is convex. By Lemma~\ref{lem:hyperspace}, $\Cred_X$ is compact, while $\mathbb R^{\mathbb N}$ with the product topology is Hausdorff. Thus the continuous injection $J$ is a homeomorphism onto its compact image.
\end{proof}

Lemma~\ref{lem:affine-embedding} shows that a compact credal set is
determined by the values of its support functional on a countable dense
family of continuous functions, but it does not characterize which
sequences of real numbers arise in this way. The following result combines
the classical support function representation of compact convex sets with
the coherence and upper envelope theory of imprecise probabilities
\citep{Radstrom1952,hormander1955support,walley1991statistical,
troffaesDeCooman2014lowerPrevisions}. It gives an explicit characterization
of the image of the countable coordinate embedding and a reconstruction
formula for the corresponding credal set, in the guise of a representation theorem. This formulation may be of
independent interest in imprecise probability.

\begin{theorem}[Countable Support Coordinate Characterization]
\label{thm:support-coordinate-characterization}
Let $X$ be a nonempty compact metric space, and let
$(u_n)_{n\geq 1}$ be norm dense in $C(X)$. Define
\[
J(\mathcal P)
=
\bigl(h_{\mathcal P}(u_n)\bigr)_{n\geq 1},
\qquad
h_{\mathcal P}(f)
=
\max_{\mu\in\mathcal P}\int_X f \mathrm{d}\mu .
\]
For $a=(a_n)_{n\geq 1}\in\mathbb R^{\mathbb N}$, the following are
equivalent.

\begin{enumerate}[label=(\roman*)]

\item There exists $\mathcal P\in\Cred_X$ such that
\[
a_n
=
\max_{\mu\in\mathcal P}\int_X u_n \mathrm{d}\mu
\qquad
\text{for every }n\geq 1.
\]

\item For every integer $m\geq 0$, every choice of indices
$n_0,n_1,\ldots,n_m$, and every
$s_0,s_1,\ldots,s_m\geq 0$,
\[
\max_{x\in X}
\left[
\sum_{i=1}^{m}
s_i\bigl(a_{n_i}-u_{n_i}(x)\bigr)
-
s_0\bigl(a_{n_0}-u_{n_0}(x)\bigr)
\right]
\geq 0.
\]

\end{enumerate}

In this case, the representing credal set is unique and is given by
\[
\mathcal P_a
=
\left\{
\mu\in\Prob(X):
\int_X u_n \mathrm{d}\mu\leq a_n
\text{ for every }n\geq 1
\right\}.
\]
In particular,
\[
a_n
=
\max_{\mu\in\mathcal P_a}\int_X u_n \mathrm{d}\mu
\qquad
\text{for every }n\geq 1,
\]
and hence
\[
J(\Cred_X)
=
\left\{
a\in\mathbb R^{\mathbb N}:
a\text{ satisfies {\rm(ii)}}
\right\}.
\]
\end{theorem}

\begin{proof}
Assume first that {\rm(i)} holds, and fix $m$, the indices
$n_0,\ldots,n_m$, and the nonnegative coefficients
$s_0,\ldots,s_m$. Since $\mathcal P$ is weakly compact and
\[
\mu\longmapsto\int_X u_{n_0} \mathrm{d}\mu
\]
is weakly continuous, there exists $\mu_0\in\mathcal P$ such that
\[
\int_X u_{n_0} \mathrm{d}\mu_0=a_{n_0}.
\]
Set
\[
G(x)
=
\sum_{i=1}^{m}
s_i\bigl(a_{n_i}-u_{n_i}(x)\bigr)
-
s_0\bigl(a_{n_0}-u_{n_0}(x)\bigr).
\]
Then
\[
\begin{aligned}
\int_X G \mathrm{d}\mu_0
&=
\sum_{i=1}^{m}
s_i
\left(
a_{n_i}-\int_Xu_{n_i} \mathrm{d}\mu_0
\right)
-
s_0
\left(
a_{n_0}-\int_Xu_{n_0} \mathrm{d}\mu_0
\right)
\geq 0,
\end{aligned}
\]
because
\[
\int_Xu_{n_i} \mathrm{d}\mu_0
\leq
\max_{\mu\in\mathcal P}\int_Xu_{n_i} \mathrm{d}\mu
=
a_{n_i}.
\]
Consequently,
\[
\max_{x\in X}G(x)
\geq
\int_XG \mathrm{d}\mu_0
\geq 0,
\]
which proves {\rm(ii)}.

Conversely, assume {\rm(ii)}. First observe that if $u_j=u_k$, then
$a_j=a_k$. Indeed, applying {\rm(ii)} with $m=1$,
$s_0=s_1=1$, $n_0=j$, and $n_1=k$ gives $a_k-a_j\geq0$;
interchanging $j$ and $k$ gives the reverse inequality. Hence the
assignment
$\overline E_a(u_n)=a_n$
defines a real functional on the set
$\mathcal U=\{u_n:n\geq1\}$.

Condition~{\rm(ii)} is precisely Walley's coherence condition for the
upper prevision $\overline E_a$ on the domain $\mathcal U$ 
\citep[Definition~2.5.1]{walley1991statistical}. By the corresponding
upper envelope representation
\citep[Section~3.3.3]{walley1991statistical}, 
\citep[Chapter~4]{troffaesDeCooman2014lowerPrevisions}, there exists a
nonempty family $\mathfrak L_a$ of linear previsions on the bounded
gambles on $X$ such that
\[
L(u_n)\leq a_n
\qquad
\text{for every }L\in\mathfrak L_a
\text{ and every }n\geq1,
\]
and
\[
a_n
=
\sup_{L\in\mathfrak L_a}L(u_n)
\qquad
\text{for every }n\geq1.
\]

The restriction of each $L\in\mathfrak L_a$ to $C(X)$ is a positive
normalized linear functional. By the Riesz-Markov representation theorem,
there is a unique $\mu_L\in\Prob(X)$ such that
\[
L(f)=\int_Xf \mathrm{d}\mu_L
\qquad
\text{for every }f\in C(X).
\]
In particular, $\mu_L\in\mathcal P_a$, so $\mathcal P_a$ is nonempty.

For each $n$, the map
$\mu\mapsto\int_Xu_n \mathrm{d}\mu$
is weakly continuous. Thus $\mathcal P_a$ is a weakly closed convex subset
of the weakly compact space $\Prob(X)$, and therefore
$\mathcal P_a\in\Cred_X$. Its definition gives
\[
\sup_{\mu\in\mathcal P_a}\int_Xu_n \mathrm{d}\mu\leq a_n.
\]
On the other hand, every $L\in\mathfrak L_a$ corresponds to a measure
$\mu_L\in\mathcal P_a$, and therefore
\[
a_n
=
\sup_{L\in\mathfrak L_a}L(u_n)
\leq
\sup_{\mu\in\mathcal P_a}\int_Xu_n\mathrm{d}\mu
\leq
a_n.
\]
Hence
\[
\sup_{\mu\in\mathcal P_a}\int_Xu_n\mathrm{d}\mu=a_n.
\]
Since $\mathcal P_a$ is weakly compact and
$\mu\mapsto\int_Xu_n\mathrm{d}\mu$ is weakly continuous, this supremum is
attained. Thus
\[
\max_{\mu\in\mathcal P_a}\int_Xu_n\mathrm{d}\mu=a_n
\]
for every $n\geq1$, proving {\rm(i)}.

It remains to prove uniqueness. Suppose that $\mathcal Q\in\Cred_X$
also satisfies $J(\mathcal Q)=a$. Then
$\mathcal Q\subseteq\mathcal P_a$ by the definition of $\mathcal P_a$.
Conversely, let $\mu\in\mathcal P_a$ and fix $f\in C(X)$. Choose a
sequence $(u_{n_k})_{k\geq1}$ such that
\[
\lVert u_{n_k}-f\rVert_\infty\longrightarrow0.
\]
Since integration against a probability measure and the support functional
$h_{\mathcal Q}$ are both $1$-Lipschitz with respect to the uniform norm,
\[
\begin{aligned}
\int_X f \mathrm{d}\mu
&=
\lim_{k\to\infty}\int_Xu_{n_k} \mathrm{d}\mu
\leq
\lim_{k\to\infty}a_{n_k}
=
\lim_{k\to\infty}h_{\mathcal Q}(u_{n_k})
=
h_{\mathcal Q}(f).
\end{aligned}
\]
Hence
\[
\int_X f \mathrm{d}\mu\leq h_{\mathcal Q}(f)
\qquad
\text{for every }f\in C(X).
\]
The support halfspace characterization of weakly closed convex subsets of
$\Prob(X)$ now implies that $\mu\in\mathcal Q$. Therefore
$\mathcal P_a\subseteq\mathcal Q$, and consequently
$\mathcal P_a=\mathcal Q$.
\end{proof}

The preceding results complete the topological and affine preliminaries for
the dynamical analysis. We now distinguish the two structures used below:
metric contraction and inclusion monotonicity. A contraction inequality asks
whether two different initial credal sets are driven closer together and,
ultimately, whether the dynamics forget their initial condition. Inclusion
monotonicity asks instead whether nested uncertainty models remain nested
under the update,
\[
  \mathcal P\subseteq\mathcal Q
  \quad\Longrightarrow\quad
  F(\mathcal P)\subseteq F(\mathcal Q).
\]
Neither property implies the other in general. The inclusion order is
particularly natural when lower and upper credal bounds are available,
because it permits comparison from the endpoints without requiring a
quantitative contraction coefficient. The inclusion structure is used in
Sections~\ref{sec:order} and~\ref{sec:varying}, while
Section~\ref{sec:imsg} returns to metric contraction in the Hausdorff metric
induced by the Wasserstein distance.

\textbf{Application to finite possibility cores.} To illustrate the utility of
Theorem~\ref{thm:support-coordinate-characterization}, we show how its
support coordinate representation recovers the classical nested set
characterization of finite possibility cores
\citep{destercke2008unifying}.

Let $Y$ be a finite set, which in machine learning models the space of labels for inputs of interest. A possibility contour is a
function $\pi:Y\to[0,1]$ assigning a degree of plausibility to each label,
with larger values indicating greater plausibility;\footnote{These values are not probabilities of the individual labels
and need not sum to one.} we assume it is normalized, so that
$\max_{y\in Y}\pi(y)=1$. The contour induces a possibility measure (a special case of an upper probability) 
$\Pi(A)=\max_{y\in A}\pi(y)$ for every nonempty $A\subseteq Y$, with
$\Pi(\emptyset)=0$. Its \emph{core} is the credal set of all probability
measures $P$ that are set-wise dominated by $\Pi$, that is, all the probabilities satisfying $P(A)\leq\Pi(A)$, for every $A\subseteq Y$; we write it as $\mathcal K_\pi
    :=
    \{
        P\in\mathcal P(Y):
        P(A)\le\Pi(A)
        \text{, for every }A\subseteq Y
    \}$.

Although this definition involves all subsets of $Y$, the core can be
characterized using only the nested sets
$C_\alpha=\{y\in Y:\pi(y)>\alpha\}$. Specifically, it suffices to require
$P(C_\alpha)\geq1-\alpha$ for each distinct value $\alpha$ attained by
$\pi$ \citep{destercke2008unifying}. The following corollary expresses
these bounds through support values at the negatives of the indicator functions of the sets $C_\alpha$, denoted as
$-\mathbf{1}_{C_\alpha}$, showing that this restricted family of
coordinates suffices to determine the entire core.

This construction also arises in conformal prediction, a popular statistical and machine learning predictive procedure \citep{shafer2008tutorial}. For a fixed input
$x$, conformal prediction assigns a p-value to each candidate label $y$.
If their maximum is below one, a label attaining the largest p-value can
be assigned value one, leaving the other values unchanged. This yields
a normalized possibility contour $\pi_x$
\citep{chau2026conformal}. For a chosen threshold $\alpha$,
the set $C_\alpha(x)=\{y\in Y:\pi_x(y)>\alpha\}$ contains the labels whose
plausibility exceeds that threshold. These prediction regions are nested
as $\alpha$ varies. The corollary then characterizes the credal set of all
probability measures assigning at least $1-\alpha$ probability to each
such region, connecting our support coordinate characterization to the
representation studied in conformal prediction
\citep[Eq.~(7)]{chau2026conformal}.

\begin{corollary}[Nested Cut Reconstruction of Finite Possibility Cores]
\label{cor:possibility-nested-cuts}
Let $Y$ be a nonempty finite label space equipped with
the discrete metric, and let $\pi:Y\to[0,1]$ be a
normalised possibility contour, so that
$\max_{y\in Y}\pi(y)=1$.
For $t\in[0,1]$, let
    $C_t:=\{y\in Y:\pi(y)>t\}$, 
    $T:=\pi(Y)\cap[0,1)$. 
Let $(u_n)_{n\ge1}$ enumerate the rational-valued functions
$\mathbb Q^{Y}$, and set
\[
    a_n:=h_{\mathcal K_\pi}(u_n)
        =\max_{P\in\mathcal K_\pi}
          \int_{Y}u_n\,dP.
\]
Then, for each $t\in T$, there is an index $n(t)$ such that
\[
    u_{n(t)}=-\mathbf 1_{C_t},
    \qquad
    a_{n(t)}=t-1.
\]
Moreover, the representing credal set $\mathcal P_a$ in
Theorem~\ref{thm:support-coordinate-characterization} satisfies
\begin{equation}
\label{eq:possibility-nested-cuts}
\begin{aligned}
    \mathcal P_a
    =\mathcal K_\pi&=\left\{
        P\in\mathcal P(Y):
        P(C_t)\ge1-t
        \text{ for every }t\in T
      \right\}\\
    &=\left\{
        P\in\mathcal P(Y):
        P(C_t)\ge1-t
        \text{ for every }t\in[0,1]
      \right\}.
\end{aligned}
\end{equation}
Thus, the finitely many support-coordinate inequalities indexed by
$\{n(t):t\in T\}$ suffice to reconstruct $\mathcal K_\pi$;
all remaining inequalities are redundant.
\end{corollary}

\begin{proof}
Choose $y^\star\in Y$ with $\pi(y^\star)=1$.
The point mass $\delta_{y^\star}$ belongs to $\mathcal K_\pi$,
so the core is nonempty.
It is also a closed convex subset of the finite probability
simplex and is therefore compact.
Since $\mathbb Q^{Y}$ is norm dense in $C(Y)$,
Theorem~\ref{thm:support-coordinate-characterization} applies
with $X=Y$ and yields $\mathcal P_a=\mathcal K_\pi$.

For every $P\in\mathcal K_\pi$ and $t\in[0,1]$,
\[
    P(C_t^c)\le\Pi(C_t^c)\le t,
\]
and hence $P(C_t)\ge1-t$.
To establish sharpness at each attained level $t\in T$,
choose $y_t\in Y$ with $\pi(y_t)=t$ and define
\[
    Q_t:=t\delta_{y_t}+(1-t)\delta_{y^\star}.
\]
This measure belongs to $\mathcal K_\pi$.
Indeed, if an event $A$ contains $y^\star$, then
$Q_t(A)\le1=\Pi(A)$; otherwise,
\[
    Q_t(A)=t\mathbf 1_A(y_t)\le\Pi(A).
\]
Since $y_t\notin C_t$ and $y^\star\in C_t$,
we have $Q_t(C_t)=1-t$.
The function $-\mathbf 1_{C_t}$ is rational-valued and therefore
appears in the sequence as $u_{n(t)}$, with support value
\[
    a_{n(t)}
    =\max_{P\in\mathcal K_\pi}\{-P(C_t)\}
    =t-1.
\]
Consequently, its support-coordinate inequality is exactly
\[
    \int_{Y}u_{n(t)}\,dP\le a_{n(t)}
    \quad\Longleftrightarrow\quad
    P(C_t)\ge1-t.
\]

Conversely, suppose $P\in\mathcal P(Y)$ satisfies
$P(C_t)\ge1-t$ for every $t\in T$.
For a nonempty event $A\subseteq Y$, set $s:=\Pi(A)$.
If $s<1$, then $s\in T$ and $A\subseteq C_s^c$, so
\[
    P(A)\le P(C_s^c)=1-P(C_s)\le s=\Pi(A).
\]
If $s=1$, the bound is automatic, as is the constraint for
the empty event.
Thus $P\in\mathcal K_\pi$.
Together with the forward implication, this proves
\eqref{eq:possibility-nested-cuts} and the claimed redundancy.
\end{proof}

Theorem~\ref{thm:support-coordinate-characterization} gives a countable
support representation of credal sets. For a possibility core on a finite
label space, Corollary~\ref{cor:possibility-nested-cuts} reduces this representation to at most $|Y|-1$
support inequalities without changing the represented set of probability
measures.
\section{Classical Consequences and a Noncontractive Example}\label{sec:classical}

\subsection{Classical fixed point specializations}
\label{subsec:classical-specializations}

The results in this subsection are not new fixed point principles. They
record the consequences of classical theorems once the compactness,
completeness, and affine embedding facts established in
Section~\ref{sec:preliminaries} have been verified for the credal
hyperspace. They provide a baseline separating fixed point existence from
uniqueness and orbit convergence.

\begin{proposition}[Schauder-Tychonoff Specialization]\label{cor:schauder}
Let $X$ be compact metric. Every Hausdorff-continuous self-map
\[
 F:\Cred_X\to\Cred_X
\]
has a fixed point. Moreover, $\Fix(F)$ is a nonempty compact subset of $\Cred_X$.\footnote{Here of course $\Fix(F)$ denotes the collection of fixed points of map $F$.}
\end{proposition}

\begin{proof}
By Lemma~\ref{lem:affine-embedding}, $J(\Cred_X)$ is a compact convex subset of a locally convex topological vector space. The conjugate map $J\circ F\circ J^{-1}$ is continuous, so Schauder-Tychonoff yields a fixed point \citep[Corollary~17.56]{aliprantis}. Pulling it back through $J$ yields a fixed point of $F$. Since $F$ is continuous,
\[
 \Fix(F)=\{\mathcal P\in\Cred_X:d_H(\mathcal P,F(\mathcal P))=0\}
\]
is a closed subset of the compact space $\Cred_X$, and hence compact.
\end{proof}

\begin{remark}
The only inverse in the preceding conjugation is $J^{-1}$ on the image $J(\Cred_X)$. No inverse of $F$ is required. Because $\Cred_X$ is compact and $\mathbb R^{\mathbb N}$ is Hausdorff, a continuous injective $J$ is automatically a homeomorphism onto its image.
\end{remark}

Existence is not an orbit convergence statement. Under strict distance decrease, however, the compact space theorem of \citet{edelstein1962} gives both uniqueness and convergence.

\begin{proposition}[Edelstein Specialization]\label{cor:edelstein}
Let $X$ be compact metric and let $F:\Cred_X\to\Cred_X$ satisfy
\[
 d_H(F(\mathcal P),F(\mathcal Q))
 <d_H(\mathcal P,\mathcal Q)
 \qquad\text{whenever }\mathcal P\neq\mathcal Q.
\]
Then, $F$ has a unique fixed point $\mathcal P_\star$, and
\[
 F^n(\mathcal P_0)\xrightarrow[d_H]{}\mathcal P_\star
 \qquad\text{for every }\mathcal P_0\in\Cred_X.
\]
\end{proposition}

\begin{proof}
This is Edelstein's compact metric space fixed point theorem \citep{edelstein1962}, applied to $(\Cred_X,d_H)$, which is compact by Lemma~\ref{lem:hyperspace}.
\end{proof}

The compactness of $X$ can be relaxed when one imposes a complete metric contraction condition. We use the classical Boyd-Wong formulation.

\begin{proposition}[Boyd-Wong Specialization]\label{cor:boyd-wong}
Let $(X,d)$ be a complete separable metric space and equip $\Prob(X)$
with the Prokhorov metric induced by $d$, and compact credal sets with
the corresponding Hausdorff metric. Let $F:\Cred_X\to\Cred_X$ be a self-map. Suppose there exists a function $\psi:[0,\infty)\to[0,\infty)$ that is upper semicontinuous from the right, satisfies $\psi(t)<t$ for every $t>0$, and obeys
\begin{equation}\label{eq:boyd-wong}
 d_H(F(\mathcal P),F(\mathcal Q))
 \le \psi\bigl(d_H(\mathcal P,\mathcal Q)\bigr)
 \qquad\text{for all }\mathcal P,\mathcal Q\in\Cred_X.
\end{equation}
Then $F$ has a unique fixed point $\mathcal P_\star$, and $F^n(\mathcal P_0)\to\mathcal P_\star$ for every $\mathcal P_0\in\Cred_X$.
\end{proposition}

\begin{proof}
By Lemma~\ref{lem:hyperspace}, the credal hyperspace is complete. The conclusion is therefore exactly the Boyd-Wong fixed point theorem \citep{boyd1969nonlinear}.
\end{proof}

Two remarks are in order.

\begin{remark}[What These Propositions Do and Do Not Say]\label{rem:classical-scope}
Proposition~\ref{cor:schauder} gives existence under continuity on a compact credal hyperspace; it does not provide an algorithm for finding a fixed point and does not imply convergence of successive updates. Propositions~\ref{cor:edelstein} and~\ref{cor:boyd-wong} are orbit convergence results, but their metric contraction assumptions can be strong. Their inequalities compare the images of two credal sets; they do not say that the diameter of a single credal set must shrink. Sections~\ref{sec:order} and~\ref{sec:varying} use the inclusion order
instead. This leads to extremal fixed points and monotone approximation
without requiring metric contraction. Section~\ref{sec:imsg} later returns
to metric contraction in a structured continuous time setting.
\end{remark}

\begin{remark}[Quantitative Convergence]
The Edelstein and Boyd-Wong conclusions are qualitative under their stated
assumptions. Strict decrease of the Hausdorff distance does not by itself
provide a uniform contraction factor, and convergence may be slower than
geometric. Quantitative estimates require additional control of the distance
decrease. In particular, if
\[
    d_H(F(\mathcal P),F(\mathcal Q))
    \leq q\,d_H(\mathcal P,\mathcal Q),
    \qquad q\in(0,1),
\]
then
\[
    d_H(F^n(\mathcal P_0),\mathcal P_\star)
    \leq q^n d_H(\mathcal P_0,\mathcal P_\star).
\]
For a general Boyd-Wong comparison function, any finer rate depends on
additional properties of that function.
\end{remark}
\subsection{Finite Family Bayesian Tilting: Continuity Without Contraction}\label{sec:bayes}

This subsection gives a concrete boundary example for the preceding
classical consequences. The update is Hausdorff-continuous on the full
credal hyperspace, so fixed point existence follows, but the available
positivity assumptions do not provide a contraction or an orbit convergence
mechanism. The construction is the measure-theoretic finite family tilt
underlying the finitely generated update in Credal Bayesian Deep Learning
\citep{caprio2024credal}; no neural network structure is needed here. We use
the term \emph{finite family Bayesian tilt} to distinguish this operator from
other conditioning rules in imprecise probability
\citep{gong2021judicious}.

Let $\Theta$ be compact metric and let $w_1,\ldots,w_K\in C(\Theta)$ satisfy
\begin{equation}\label{eq:positive-likelihood}
 w_k(\theta)>0
 \qquad\text{for every }\theta\in\Theta,\quad k=1,\ldots,K.
\end{equation}
For example, after observing an event $E$ in an evidence space, $w_k(\theta)$ may arise as the likelihood $L_k(E\mid\theta)$; in this section we assume this function of $\theta$ is continuous and strictly positive. Compactness then implies
\[
 0<\min_{\theta\in\Theta}w_k(\theta)
 \le \max_{\theta\in\Theta}w_k(\theta)<\infty.
\]
For $P\in\Prob(\Theta)$, define the Bayesian tilt
\begin{equation}\label{eq:bayes-tilt}
 T_k(P)(A)
 \coloneqq \frac{\int_Aw_k(\theta)\,P(\mathrm d\theta)}
         {\int_\Theta w_k(\theta)\,P(\mathrm d\theta)},
 \qquad A\in\Borel(\Theta).
\end{equation}
The denominator is positive by \eqref{eq:positive-likelihood}.

\begin{lemma}[Continuity of Bayesian Tilt]\label{lem:bayes-continuity}
Each $T_k:\Prob(\Theta)\to\Prob(\Theta)$ is weakly continuous.
\end{lemma}

\begin{proof}
Suppose $P_n\Rightarrow P$; here $\Rightarrow$ denotes weak convergence of
probability measures, while $\longrightarrow$ below denotes convergence in
$\mathbb R$. For any $h\in C(\Theta)$,
\[
 \int h\,\mathrm dT_k(P_n)
 =\frac{\int h w_k\,\mathrm dP_n}{\int w_k\,\mathrm dP_n}.
\]
Both $h w_k$ and $w_k$ are continuous and bounded. Weak convergence therefore gives convergence of numerator and denominator, and the limiting denominator is strictly positive. Hence
\[
 \int h\,\mathrm dT_k(P_n)\longrightarrow\int h\,\mathrm dT_k(P),
\]
which is exactly $T_k(P_n)\Rightarrow T_k(P)$.
\end{proof}

For a credal set $\mathcal P\in\Cred_\Theta$, define the finite family
Bayesian convex hull update by
\begin{equation}\label{eq:credal-bayes}
 \mathsf B(\mathcal P)
 \coloneqq \clco\left(\bigcup_{k=1}^K T_k[\mathcal P]\right).
\end{equation}
This is a self-map of the full credal hyperspace $\Cred_\Theta$, not merely
of its finitely generated subfamily.

Although each $T_k$ is a Bayesian update for the likelihood $w_k$, iteration
of $\mathsf B$ should not automatically be interpreted as repeated
conditioning on one fixed item of evidence. Reapplying a fixed tilt compounds
the likelihood, producing a factor $w_k^n$ after $n$ iterations. A sequential
Bayesian interpretation therefore requires the successive tilts to represent
successive likelihood contributions under an appropriate data generating
model. Here the operator is used primarily as a concrete case study for the
topological fixed point framework.

\begin{proposition}[Hausdorff Continuity of the Bayesian Convex Hull Update]\label{prop:bayes-hausdorff}
Under \eqref{eq:positive-likelihood}, the operator
\[
 \mathsf B:\Cred_\Theta\to\Cred_\Theta
\]
defined in \eqref{eq:credal-bayes} is Hausdorff-continuous.
\end{proposition}

\begin{proof}
By Lemma~\ref{lem:bayes-continuity}, each $T_k$ is continuous on compact $\Prob(\Theta)$ and hence uniformly continuous. Therefore the induced image map
\[
 \mathcal P\longmapsto T_k[\mathcal P]
\]
is Hausdorff-continuous on nonempty compact subsets of $\Prob(\Theta)$. A finite union of Hausdorff-continuous compact set images is Hausdorff-continuous. Applying Lemma~\ref{lem:clco} then shows that \eqref{eq:credal-bayes} is Hausdorff-continuous. Its values are nonempty compact convex subsets of $\Prob(\Theta)$, so $\mathsf B$ is a self-map of $\Cred_\Theta$.
\end{proof}

Proposition~\ref{prop:bayes-hausdorff} verifies one topological ingredient of
the abstract theory, and it is useful to be precise about what follows.
Because $\mathsf B$ is a Hausdorff-continuous self-map of the compact credal
hyperspace, Proposition~\ref{cor:schauder} gives fixed point existence.
Because $\mathsf B$ is inclusion-monotone, the order theory of
Section~\ref{sec:order} becomes available once an invariant credal interval
$[\mathcal A,\mathcal B]$ has been identified. Neither continuity nor
monotonicity alone supplies metric contraction or identifies such an
interval.

Thus the case study separates verification of an update's structural
properties from the dynamical conclusions that those properties justify.
Before discussing the weakness of fixed point existence for this particular
operator, we show that the full domain formulation recovers exactly the
finitely generated update used computationally.



\begin{proposition}[Exact Reduction for Finitely Generated Credal Sets]\label{prop:fgcs}
Let
\[
 \mathcal P=\co(\{P_1,\ldots,P_M\})\in\Cred_\Theta.
\]
Then, for each $k$,
\begin{equation}\label{eq:tilt-convex-hull}
 T_k[\mathcal P]
 =\co(\{T_k(P_1),\ldots,T_k(P_M)\}),
\end{equation}
and therefore
\begin{equation}\label{eq:fgcs-update}
 \mathsf B(\mathcal P)
 =\co(\{T_k(P_i):i \in \{1,\ldots,M\}, k\in \{1,\ldots,K\}\}).
\end{equation}
\end{proposition}

\begin{proof}
Write $z_i\coloneqq \int_\Theta w_k\,\mathrm dP_i>0$. If $P=\sum_{i=1}^M\alpha_iP_i$ with $\alpha_i\ge0$ and $\sum_i\alpha_i=1$, then
\begin{equation}\label{eq:posterior-mixture-weights}
 T_k(P)
 =\sum_{i=1}^M
   \frac{\alpha_i z_i}{\sum_{j=1}^M\alpha_jz_j}\,T_k(P_i).
\end{equation}
Thus $T_k[\mathcal P]$ is contained in the right hand side of \eqref{eq:tilt-convex-hull}. Conversely, for any probability vector $(\beta_i)_{i=1}^M$, set
\[
 \alpha_i
 =\frac{\beta_i/z_i}{\sum_{j=1}^M\beta_j/z_j}.
\]
The posterior weights in \eqref{eq:posterior-mixture-weights} are then exactly $\beta_i$. Hence every convex combination of the $T_k(P_i)$ is attained, proving \eqref{eq:tilt-convex-hull}. Taking the convex hull over the finite family $k\in\{1,\ldots,K\}$ gives \eqref{eq:fgcs-update}.
\end{proof}

\begin{remark}[Why Fixed Point Existence Alone is Weak Here]\label{rem:bayes-fixed-point}
Proposition~\ref{prop:bayes-hausdorff} is a continuity statement, not a claim that repeated fixed evidence Bayesian updating is a canonical learning protocol. The existence conclusion can also be uninformative for this operator. For example, each $T_k$ is a bijection of $\Prob(\Theta)$ with inverse given by tilting with $1/w_k$, so
\[
 \mathsf B(\Prob(\Theta))=\Prob(\Theta).
\]
Thus the vacuous credal set is already a fixed point. This recovers the
familiar observation that pointwise Bayesian updating of the class of all
priors does not produce learning: complete prior ignorance remains vacuous
after observing the data
\citep[Page 93]{walley1991statistical}. Dirac singletons are fixed as well, because every positive tilt leaves a point mass unchanged. A practically meaningful convergence claim therefore requires additional structure beyond the continuity assumptions of Proposition~\ref{prop:bayes-hausdorff}.
\end{remark}

Hilbert's projective metric is a classical tool in the contraction theory of
positive operators \citep{bushell1973hilbert}. It is therefore tempting to
interpret uniform positive upper and lower bounds on a likelihood as a route
to projective contraction of Bayesian updating. For a common likelihood tilt,
however, the normalization has a simple invariance that prevents strict
contraction.

\begin{remark}[Common Bayesian Tilt is Projectively Isometric]\label{rem:hilbert}
Positivity bounds on a fixed likelihood do not make Bayesian tilt a strict contraction in Hilbert's projective metric. Suppose $P$ and $Q$ have strictly positive densities $p$ and $q$ with respect to a common measure $\mu$, and suppose $0<\operatorname*{ess\,inf}(p/q)\leq\operatorname*{ess\,sup}(p/q)<\infty$. Define
\[
 d_{\mathrm{Hil}}(p,q)
 \coloneqq \log\frac{\operatorname*{ess\,sup}(p/q)}
                  {\operatorname*{ess\,inf}(p/q)}.
\]
If $T_w$ denotes the common tilt by a strictly positive function $w$ with $0<Pw,Qw<\infty$ (in particular, one of the functions $w_k$ above), then the tilted densities are
\[
 p'=\frac{wp}{Pw},\qquad q'=\frac{wq}{Qw},
\]
so
\[
 \frac{p'}{q'}=\frac{Qw}{Pw}\frac pq.
\]
The scalar $Qw/Pw$ cancels between the essential supremum and infimum. Hence
\begin{equation}\label{eq:hilbert-isometry}
 d_{\mathrm{Hil}}(p',q')=d_{\mathrm{Hil}}(p,q).
\end{equation}
Thus a common positive Bayesian tilt is a projective isometry. A strict projective contraction would require an additional mechanism that changes projective ratios---for example, a suitable mixing or regularization step; upper and lower likelihood bounds alone do not suffice.
\end{remark}


\section{Order Intervals, Extremal Fixed Points, and Invariant Hulls}\label{sec:order}
We now turn from the topological and metric criteria of
Section~\ref{sec:classical} to a complementary order framework. Set
inclusion on credal intervals will yield extremal fixed points and monotone
endpoint dynamics without requiring a contraction hypothesis.
Let $X$ be compact metric and let $\mathcal A,\mathcal B\in\Cred_X$ satisfy $\mathcal A\subseteq\mathcal B$. Define the credal order interval
\begin{equation}\label{eq:order-interval}
 [\mathcal A,\mathcal B]
 \coloneqq \{\mathcal P\in\Cred_X:\mathcal A\subseteq\mathcal P\subseteq\mathcal B\},
\end{equation}
ordered by set inclusion. We first record the fact that, when $\mathcal A \subsetneq \mathcal B$, the interval $[\mathcal A,\mathcal B]$ contains genuinely intermediate credal sets.

\begin{remark}[Nontriviality of Credal Order Intervals]
\label{rem:order-interval-nontrivial}
The interval $[\mathcal A,\mathcal B]$ always contains its two endpoints.
If $\mathcal A\subsetneq\mathcal B$, it also contains a continuum of
intermediate credal sets. Indeed, for $t\in(0,1)$, let
$\mathcal C_t
\coloneqq 
(1-t)\mathcal A+t\mathcal B$.
Then $\mathcal C_t\in\Cred_X$ and
$\mathcal A\subseteq\mathcal C_t\subseteq\mathcal B$.
Choose $u\in C(X)$ such that
$h_{\mathcal A}(u)<h_{\mathcal B}(u)$,
which is possible because $\mathcal A\subsetneq\mathcal B$. Since
\[
h_{\mathcal C_t}(u)
=
(1-t)h_{\mathcal A}(u)+t h_{\mathcal B}(u),
\]
we have
\[
h_{\mathcal A}(u)
<
h_{\mathcal C_t}(u)
<
h_{\mathcal B}(u).
\]
Thus
$\mathcal A\subsetneq\mathcal C_t\subsetneq\mathcal B$.
The same support identity shows that
$\mathcal C_s\neq\mathcal C_t$ whenever $s\neq t$.
\end{remark}

By Lemma~\ref{lem:hyperspace}.(c), $[\mathcal A,\mathcal B]$ is closed in the compact space $\Cred_X$, and hence is compact in the Hausdorff metric. We then show that the interval $[\mathcal A,\mathcal B]$ with set inclusion $\subseteq$ is a complete lattice.

\begin{lemma}[Credal Order Intervals are Complete Lattices]\label{lem:lattice}
The partially ordered set $([\mathcal A,\mathcal B],\subseteq)$ is a complete lattice. For a nonempty family $(\mathcal P_i)_{i\in I}\subseteq[\mathcal A,\mathcal B]$,
\begin{align*}
 \bigwedge_{i\in I}\mathcal P_i&=\bigcap_{i\in I}\mathcal P_i,\\
 \bigvee_{i\in I}\mathcal P_i&=\clco\left(\bigcup_{i\in I}\mathcal P_i\right).
\end{align*}
The meet of the empty family is $\mathcal B$ and the join of the empty family is $\mathcal A$.
\end{lemma}

\begin{proof}
Every $\mathcal P_i$ contains $\mathcal A$, so the intersection is nonempty. It is weakly closed and convex and, as a closed subset of compact $\mathcal B$, is compact; it lies between $\mathcal A$ and $\mathcal B$ and is plainly the greatest lower bound. The closed convex hull of the union contains every $\mathcal P_i$. Because $\mathcal B$ is closed and convex and contains the union, this hull is a weakly closed convex subset of compact $\mathcal B$, hence is itself compact; it is the least upper bound.
\end{proof}

Recall that a self-map $F:[\mathcal A,\mathcal B]\to[\mathcal A,\mathcal B]$ is \emph{inclusion-monotone} if
\[
 \mathcal P\subseteq\mathcal Q\quad\Longrightarrow\quad F(\mathcal P)\subseteq F(\mathcal Q).
\]

\begin{proposition}[Knaster-Tarski on a Credal Interval]\label{cor:tarski}
Every inclusion-monotone self-map $F$ of $[\mathcal A,\mathcal B]$ has a least fixed point $\mathcal P_-$ and a greatest fixed point $\mathcal P_+$. In fact, the set of fixed points of $F$ is itself a complete lattice.
\end{proposition}

\begin{proof}
The conclusion follows from the Knaster-Tarski fixed point theorem \citep{tarski1955}, because Lemma~\ref{lem:lattice} gives a complete lattice.
\end{proof}

Proposition~\ref{cor:tarski} is an existence result and does not, by itself,
say that ordinary sequential iteration reaches either extremal fixed point.
On a general complete lattice, constructive fixed point procedures may require
transfinite chains \citep{cousot1979constructive}. The compact credal setting
provides additional topology. The increasing lower endpoint chain and the
decreasing upper endpoint chain remain in the compact set
$[\mathcal A,\mathcal B]$, while Hausdorff continuity permits the update map to
be passed through their limits. The following theorem shows that these two
ingredients are sufficient for ordinary countable iteration.

\begin{theorem}[Endpoint Iteration]\label{thm:endpoint}
Let $X$ be compact metric and let
 $F:[\mathcal A,\mathcal B]\to[\mathcal A,\mathcal B]$
be inclusion-monotone and Hausdorff-continuous. Define
\[
 \mathcal L_0=\mathcal A,\quad \mathcal L_{n+1}=F(\mathcal L_n), \quad \mathcal U_0=\mathcal B,\quad \mathcal U_{n+1}=F(\mathcal U_n).
\]
Then
\[
 \mathcal L_0\subseteq\mathcal L_1\subseteq\cdots,
 \qquad
 \mathcal U_0\supseteq\mathcal U_1\supseteq\cdots,
\]
and there exist $\mathcal P_-,\mathcal P_+\in[\mathcal A,\mathcal B]$ such that
\[
 \mathcal L_n\xrightarrow[d_H]{}\mathcal P_- ,
 \qquad
 \mathcal U_n\xrightarrow[d_H]{}\mathcal P_+.
\]
The limits are respectively the least and greatest fixed points of $F$.
\end{theorem}

\begin{proof}
Since $F(\mathcal A)\in[\mathcal A,\mathcal B]$, $\mathcal A\subseteq F(\mathcal A)$. Monotonicity then yields $\mathcal L_n\subseteq\mathcal L_{n+1}$ for every $n$. Similarly, $F(\mathcal B)\subseteq\mathcal B$ and $\mathcal U_{n+1}\subseteq\mathcal U_n$.

Set
\[
 \mathcal L_\infty=\overline{\bigcup_{n\ge0}\mathcal L_n}.
\]
Because the $\mathcal L_n$ are nested convex subsets of the compact set $\mathcal B$, $\mathcal L_\infty$ is a credal set in the interval. We claim $d_H(\mathcal L_n,\mathcal L_\infty)\to0$. If not, there are $\varepsilon>0$, indices $n_k\uparrow\infty$, and $P_k\in\mathcal L_\infty$ such that $\dist_\rho(P_k,\mathcal L_{n_k})\ge\varepsilon$. By compactness, pass to a subsequence with $P_k\to P\in\mathcal L_\infty$. Choose $m$ and $Q\in\mathcal L_m$ with $\rho(P,Q)<\varepsilon/3$. For all sufficiently large $k$, $n_k\ge m$ and $\rho(P_k,P)<\varepsilon/3$, so
\[
 \dist_\rho(P_k,\mathcal L_{n_k})
 \le\rho(P_k,P)+\rho(P,Q)<2\varepsilon/3,
\]
a contradiction. Thus $\mathcal L_n\to\mathcal L_\infty$. Continuity gives
\[
 F(\mathcal L_\infty)
 =\lim_{n\to\infty}F(\mathcal L_n)
 =\lim_{n\to\infty}\mathcal L_{n+1}
 =\mathcal L_\infty.
\]
If $\mathcal P$ is any fixed point, then $\mathcal A\subseteq\mathcal P$ and induction gives $\mathcal L_n\subseteq\mathcal P$ for all $n$. Hence $\mathcal L_\infty\subseteq\mathcal P$, so it is the least fixed point.

For the upper sequence, set $\mathcal U_\infty=\bigcap_{n\ge0}\mathcal U_n$, which is nonempty because every $\mathcal U_n$ contains $\mathcal A$. If $\mathcal U_n$ did not converge to $\mathcal U_\infty$, there would be $\varepsilon>0$, indices $n_k\uparrow\infty$, and $P_k\in\mathcal U_{n_k}$ with $\dist_\rho(P_k,\mathcal U_\infty)\ge\varepsilon$. Compactness gives a subsequence, not relabeled, converging to $P$. For every fixed $m$, all sufficiently late $P_k$ lie in the closed set $\mathcal U_m$, so $P\in\mathcal U_m$ for every $m$, hence $P\in\mathcal U_\infty$, a contradiction. Continuity makes $\mathcal U_\infty$ fixed, and comparison with arbitrary fixed points shows that it is the greatest one.
\end{proof}

Two remarks are in order.

\begin{remark}[Endpoint Squeezing]\label{rem:endpoint-squeezing}
The two endpoint limits also quantify the possible nonuniqueness of the
long run state. For every $\mathcal P_0\in[\mathcal A,\mathcal B]$,
monotonicity gives
\[
  \mathcal L_n
  \subseteq F^n(\mathcal P_0)
  \subseteq \mathcal U_n.
\]
Hence every fixed point lies between $\mathcal P_-$ and $\mathcal P_+$.
If $\mathcal P_-=\mathcal P_+\eqqcolon \mathcal P_\star$, the same Hausdorff
squeezing argument used later in Corollary~\ref{cor:coalescence} shows that
\[
  F^n(\mathcal P_0)\xrightarrow[d_H]{}\mathcal P_\star
  \qquad
  \text{for every }\mathcal P_0\in[\mathcal A,\mathcal B].
\]
Thus coalescence of the extremal fixed points upgrades endpoint convergence
to convergence of every orbit in the interval. When the extremal fixed points
are distinct, Theorem~\ref{thm:endpoint} still identifies the smallest order
interval containing all fixed points, but it does not assert convergence of
arbitrary intermediate orbits.
\end{remark}

\begin{remark}[Localization of Compactness]
The endpoint convergence argument does not require the entire space
$\Prob(X)$ to be compact. It is enough that the upper bound $\mathcal B$ be
compact, since every credal set in $[\mathcal A,\mathcal B]$ is then a closed
subset of the same compact set.
\end{remark}

Theorem~\ref{thm:endpoint} is particularly natural for invariant credal hulls. In the following, for $\mathcal P\subseteq\Prob(X)$, we write
$S[\mathcal P]\coloneqq\{S(P):P\in\mathcal P\}$.

\begin{proposition}[Smallest Invariant Credal Hull]\label{prop:invariant-hull}
Let $X$ be compact metric and $S:\Prob(X)\to\Prob(X)$ be continuous with respect to the weak topology. Suppose $\mathcal A\subseteq\mathcal B$ are credal sets and $S[\mathcal B]\subseteq\mathcal B$. Define
\begin{equation}\label{eq:invariant-hull-map}
 F_{\mathcal A}(\mathcal P)
 \coloneqq \clco\bigl(\mathcal A\cup S[\mathcal P]\bigr),
 \qquad \mathcal P\in[\mathcal A,\mathcal B].
\end{equation}
Then $F_{\mathcal A}$ is an inclusion-monotone Hausdorff-continuous self-map of $[\mathcal A,\mathcal B]$. Its least fixed point $\mathcal P_\star$ is the smallest closed convex set $\mathcal P$ satisfying
\[
 \mathcal A\subseteq\mathcal P\subseteq\mathcal B,
 \qquad S[\mathcal P]\subseteq\mathcal P.
\]
Moreover, $F_{\mathcal A}^n(\mathcal A)\to\mathcal P_\star$ in $d_H$.
If $S$ is affine, then
\begin{equation}\label{eq:affine-invariant-iterate}
 F_{\mathcal A}^n(\mathcal A)
 =\clco\left(\bigcup_{j=0}^{n}S^j[\mathcal A]\right),
\end{equation}
and therefore
\begin{equation}\label{eq:affine-invariant-limit}
 \mathcal P_\star
 =\clco\left(\bigcup_{j=0}^{\infty}S^j[\mathcal A]\right).
\end{equation}
\end{proposition}

\begin{proof}
If $\mathcal P\in[\mathcal A,\mathcal B]$, then $S[\mathcal P]\subseteq S[\mathcal B]\subseteq\mathcal B$. Since $\mathcal B$ is closed and convex, $F_{\mathcal A}(\mathcal P)$ lies in the interval. Inclusion monotonicity is immediate.

Because $\Prob(X)$ is compact, $S$ is uniformly continuous. Hence the induced image map $\mathcal P\mapsto S[\mathcal P]$ is Hausdorff-continuous on nonempty compact subsets. Union with the fixed compact set $\mathcal A$ is Hausdorff-continuous, and Lemma~\ref{lem:clco} shows that the closed convex hull operation is nonexpansive. Thus $F_{\mathcal A}$ is Hausdorff-continuous.

Theorem~\ref{thm:endpoint} gives a least fixed point $\mathcal P_\star$ and convergence from $\mathcal A$. Since $\mathcal P_\star=\clco(\mathcal A\cup S[\mathcal P_\star])$, it contains $\mathcal A$ and is forward invariant under $S$. Conversely, if $\mathcal P\in[\mathcal A,\mathcal B]$ is closed and convex,
contains $\mathcal A$, and satisfies
$S[\mathcal P]\subseteq\mathcal P$, then
$F_{\mathcal A}(\mathcal P)\subseteq\mathcal P$. Inductively, $F_{\mathcal A}^n(\mathcal A)\subseteq\mathcal P$ for all $n$, and the Hausdorff limit is also contained in $\mathcal P$. Hence $\mathcal P_\star$ is the smallest such invariant credal set within $\mathcal B$.

If $S$ is affine and continuous, then for every $C\subseteq\Prob(X)$,
\[
 S[\clco(C)]=\clco(S[C]).
\]
Indeed, continuity and affinity give the inclusion from left to right by approximating points of $\clco(C)$ with finite convex combinations from $C$. Conversely, $S[\clco(C)]$ is a compact convex set containing $S[C]$, so it contains $\clco(S[C])$. Applying this identity to \eqref{eq:invariant-hull-map} proves \eqref{eq:affine-invariant-iterate} by induction. The nested union of the sets in \eqref{eq:affine-invariant-iterate} is convex and contains every $S^j[\mathcal A]$; its closure is therefore exactly the closed convex hull in \eqref{eq:affine-invariant-limit}. Taking the Hausdorff limit and using Theorem~\ref{thm:endpoint} proves the formula.
\end{proof}

A related closed convex orbit hull formula appears in the affine setting
considered by Heitzig~\citep{heitzig2026oracles}.
Proposition~\ref{prop:invariant-hull} instead starts from a prescribed
nonempty credal set $\mathcal A$ and treats a general continuous transition
$S$, with the explicit orbit hull formula recovered when $S$ is affine.
More generally, it identifies the least fixed point with the smallest closed
convex uncertainty class between $\mathcal A$ and $\mathcal B$ that is
forward invariant under $S$. Such invariant set questions arise naturally
in imprecise stochastic dynamics and monotone dynamical systems
\citep{hartfiel1998markovSetChains,skulj2013invariantDistributions,
deCoomanDeBockLopatatzidis2016,akianGaubertLemmens2011}.

\section{Changing Update Maps and Invariant Envelopes}\label{sec:varying}

Sections~\ref{sec:classical} and~\ref{sec:order} concern autonomous dynamics:
the same map is iterated at every step. Evidence-driven or adaptive procedures
are often nonautonomous, because the realized update can change with time.
Rather than assigning a fixed point to a changing sequence of maps, we ask a
comparison question: can every admissible update path be bounded by two
homogeneous credal dynamics?

To formalize this, fix a nonempty family of inclusion-monotone admissible
updates
\[
  \{F_e:[\mathcal A,\mathcal B]\to[\mathcal A,\mathcal B]:
  e\in\mathcal E\}.
\]
The family $\mathcal E$ describes the allowed update mechanisms; the realized
index $e_n$ may change arbitrarily from one step to the next. In particular,
the construction below does not model repeated use of one fixed item of
evidence. An arbitrary sequence $(e_n)_{n\geq1}$ generates the inhomogeneous
orbit
\begin{equation}\label{eq:inhomogeneous-orbit}
 \mathcal P_n
 =F_{e_n}(\mathcal P_{n-1}),
 \qquad n\ge1.
\end{equation}
We call the resulting dynamics inhomogeneous because the update applied at
time $n$ is $F_{e_n}$ and may depend on $n$, rather than being given by a
single fixed self-map. Define the lower and upper envelope maps
\begin{align}
 F_\wedge(\mathcal P)
 &\coloneqq \bigcap_{e\in\mathcal E}F_e(\mathcal P),\label{eq:lower-envelope}\\
 F_\vee(\mathcal P)
 &\coloneqq \clco\left(\bigcup_{e\in\mathcal E}F_e(\mathcal P)\right).\label{eq:upper-envelope}
\end{align}
Unlike an unrestricted intersection of credal sets, the intersection in \eqref{eq:lower-envelope} is automatically nonempty: every $F_e(\mathcal P)$ belongs to $[\mathcal A,\mathcal B]$ and therefore contains the common lower bound $\mathcal A$. Both envelope maps are inclusion-monotone self-maps of the same interval.

\begin{proposition}[Finite Time Envelope Bounds]\label{prop:finite-envelope}
For every initial $\mathcal P_0\in[\mathcal A,\mathcal B]$, every sequence $(e_n)$, and every $n\ge0$,
\begin{equation}\label{eq:same-start-sandwich}
 F_\wedge^n(\mathcal P_0)
 \subseteq \mathcal P_n
 \subseteq F_\vee^n(\mathcal P_0).
\end{equation}
Consequently,
\begin{equation}\label{eq:endpoint-sandwich}
 F_\wedge^n(\mathcal A)
 \subseteq \mathcal P_n
 \subseteq F_\vee^n(\mathcal B).
\end{equation}
\end{proposition}

\begin{proof}
The claim is trivial for $n=0$. Suppose the left inclusion in \eqref{eq:same-start-sandwich} holds at time $n-1$. Then, using the definition of $F_\wedge$ and monotonicity of $F_{e_n}$,
\[
 F_\wedge^n(\mathcal P_0)
 =F_\wedge\bigl(F_\wedge^{n-1}(\mathcal P_0)\bigr)
 \subseteq F_{e_n}\bigl(F_\wedge^{n-1}(\mathcal P_0)\bigr)
 \subseteq F_{e_n}(\mathcal P_{n-1})
 =\mathcal P_n.
\]
The upper inclusion is analogous. Since $\mathcal A\subseteq\mathcal P_0\subseteq\mathcal B$ and the envelope maps are monotone, \eqref{eq:endpoint-sandwich} follows.
\end{proof}

Proposition~\ref{prop:finite-envelope} is pathwise and in finite time. No
convergence assumption has been made, and no continuity of the individual
maps $F_e$ is required. The role of the envelope maps is to separate the
possibly irregular inhomogeneous dynamics from two homogeneous comparison
systems. Topological assumptions enter only when we study the limits of these
comparison systems. In particular, Theorem~\ref{thm:cluster-envelope} below
requires Hausdorff continuity of $F_\wedge$ and $F_\vee$, not of every
admissible update $F_e$.

The orbit in \eqref{eq:inhomogeneous-orbit} need not converge. Compactness
nevertheless guarantees Hausdorff cluster points. The next theorem proves a
stronger conclusion: every admissible trajectory approaches a common
limiting credal interval uniformly over the initial credal set and the
switching sequence, even when the trajectory itself has no limit.

\begin{theorem}[Uniform Asymptotic Trapping]
\label{thm:cluster-envelope}
Under the standing assumptions of this section, suppose that
$F_\wedge$ and $F_\vee$ are Hausdorff-continuous. Let
\[
\mathcal L_n=F_\wedge^n(\mathcal A),
\qquad
\mathcal U_n=F_\vee^n(\mathcal B).
\]
Then
\[
\mathcal L_n\xrightarrow[d_H]{}\mathcal L_\star,
\qquad
\mathcal U_n\xrightarrow[d_H]{}\mathcal U_\star,
\]
where $\mathcal L_\star$ is the least fixed point of $F_\wedge$ and
$\mathcal U_\star$ is the greatest fixed point of $F_\vee$.
Every inhomogeneous orbit \eqref{eq:inhomogeneous-orbit} has at least one
Hausdorff cluster point, and every such cluster point $\mathcal P$ satisfies
\begin{equation}
\label{eq:cluster-trap}
\mathcal L_\star\subseteq\mathcal P\subseteq\mathcal U_\star.
\end{equation}

More quantitatively, every inhomogeneous orbit satisfies, for every
$n\geq0$,
\begin{equation}
\label{eq:quantitative-trap}
\inf_{\mathcal Q\in[\mathcal L_\star,\mathcal U_\star]}
d_H(\mathcal P_n,\mathcal Q)
\leq
\max\left\{
d_H(\mathcal L_n,\mathcal L_\star),
d_H(\mathcal U_n,\mathcal U_\star)
\right\}.
\end{equation}
The right-hand side tends to zero and is independent of the initial credal
set and the switching sequence. Consequently, every orbit approaches the
limiting trapping interval uniformly over all
$\mathcal P_0\in[\mathcal A,\mathcal B]$ and all admissible switching
sequences.
\end{theorem}

\begin{proof}
The endpoint convergence statements follow from
Theorem~\ref{thm:endpoint}, applied to $F_\wedge$ and $F_\vee$.
Proposition~\ref{prop:finite-envelope} gives
\[
\mathcal L_n\subseteq\mathcal P_n\subseteq\mathcal U_n
\]
for every inhomogeneous orbit and every $n\geq0$. In particular,
$\mathcal L_n\subseteq\mathcal U_n$. Since inclusion is closed under
Hausdorff limits by Lemma~\ref{lem:hyperspace}(c), it follows that
\[
\mathcal L_\star\subseteq\mathcal U_\star.
\]
Thus the limiting interval
$[\mathcal L_\star,\mathcal U_\star]$ is nonempty.
Fix an inhomogeneous orbit and $n\geq0$, and put
\[
\varepsilon_n
 \coloneqq 
\max\left\{
d_H(\mathcal L_n,\mathcal L_\star),
d_H(\mathcal U_n,\mathcal U_\star)
\right\}.
\]
Define
\[
\mathcal R_n
 \coloneqq 
\left\{
Q\in\mathcal U_\star:
\dist_\rho(Q,\mathcal P_n)\leq\varepsilon_n
\right\}.
\]
The set $\mathcal R_n$ is compact. Since
$\mathcal P_n\subseteq\mathcal U_n$, every $P\in\mathcal P_n$ lies within
$\varepsilon_n$ of some $Q\in\mathcal U_\star$. Such a $Q$ belongs to
$\mathcal R_n$, and hence
\[
\sup_{P\in\mathcal P_n}
\dist_\rho(P,\mathcal R_n)
\leq\varepsilon_n.
\]
By the definition of $\mathcal R_n$,
\[
\sup_{Q\in\mathcal R_n}
\dist_\rho(Q,\mathcal P_n)
\leq\varepsilon_n.
\]
Moreover, because $\mathcal L_n\subseteq\mathcal P_n$,
\[
\sup_{Q\in\mathcal L_\star}
\dist_\rho(Q,\mathcal P_n)
\leq
\sup_{Q\in\mathcal L_\star}
\dist_\rho(Q,\mathcal L_n)
\leq\varepsilon_n.
\]
It follows that
\[
d_H\bigl(
\mathcal P_n,
\mathcal L_\star\cup\mathcal R_n
\bigr)
\leq\varepsilon_n.
\]

Now define
\[
\widehat{\mathcal P}_n
 \coloneqq 
\clco\bigl(
\mathcal L_\star\cup\mathcal R_n
\bigr).
\]
Since $\mathcal R_n\subseteq\mathcal U_\star$ and
$\mathcal U_\star$ is closed and convex,
\[
\widehat{\mathcal P}_n
\in
[\mathcal L_\star,\mathcal U_\star].
\]
Furthermore, $\mathcal P_n$ is already closed and convex, so
Lemma~\ref{lem:clco} gives
\[
\begin{aligned}
d_H(\mathcal P_n,\widehat{\mathcal P}_n)
=
d_H\Bigl(
\clco(\mathcal P_n),
\clco(\mathcal L_\star\cup\mathcal R_n)
\Bigr)
\leq
d_H\bigl(
\mathcal P_n,
\mathcal L_\star\cup\mathcal R_n
\bigr)
\leq\varepsilon_n.
\end{aligned}
\]
Taking the infimum over
$\mathcal Q\in[\mathcal L_\star,\mathcal U_\star]$ proves
\eqref{eq:quantitative-trap}. Since $\varepsilon_n\to0$ and depends only on
the two endpoint sequences, the convergence to the limiting trapping
interval is uniform over all initial credal sets and all admissible switching
sequences.

Finally, $[\mathcal A,\mathcal B]$ is a closed subset of the compact space
$\Cred_X$ by Lemma~\ref{lem:hyperspace}(c), and is therefore compact. Hence
every inhomogeneous orbit has a Hausdorff-convergent subsequence. The interval
$[\mathcal L_\star,\mathcal U_\star]$ is closed, and
\eqref{eq:quantitative-trap} shows that the distance of the orbit from this
interval tends to zero. Every Hausdorff cluster point must therefore belong
to $[\mathcal L_\star,\mathcal U_\star]$, which proves
\eqref{eq:cluster-trap}.
\end{proof}

Theorem~\ref{thm:cluster-envelope} gives a uniform asymptotic bound for every admissible update path, but it does not by itself imply convergence of an inhomogeneous orbit. Different subsequences may converge to different credal sets within the limiting interval $[\mathcal L_\star,\mathcal U_\star]$. A simple sufficient condition for convergence of every orbit is that this interval collapse to a single credal set. In that case, the finite time envelope bounds squeeze every admissible trajectory to the same limit.

\begin{corollary}[Coalescing Envelopes Imply Uniform Convergence]
\label{cor:coalescence}
Under the assumptions of Theorem~\ref{thm:cluster-envelope}, if
$\mathcal L_\star=\mathcal U_\star\eqqcolon \mathcal P_\star$,
then every orbit generated by every admissible sequence $(e_n)$ and every
initial credal set $\mathcal P_0\in[\mathcal A,\mathcal B]$ converges to
$\mathcal P_\star$ in Hausdorff distance. More precisely,
\[
d_H(\mathcal P_n,\mathcal P_\star)
\leq
\max\bigl\{
d_H(\mathcal L_n,\mathcal P_\star),
d_H(\mathcal U_n,\mathcal P_\star)
\bigr\},
\]
and the right hand side tends to zero independently of the initial credal
set and the switching sequence.
\end{corollary}

\begin{proof}
Under the coalescence assumption,
$[\mathcal L_\star,\mathcal U_\star]
=
\{\mathcal P_\star\}$.
Hence the left hand side of \eqref{eq:quantitative-trap} is
$d_H(\mathcal P_n,\mathcal P_\star)$, and the conclusion follows directly.
\end{proof}

A remark is in order.

\begin{remark}[What Envelope Coalescence Entails]
The equality $\mathcal L_\star=\mathcal U_\star$ can be interpreted as robust
loss of memory of the update path: every admissible switching sequence and
every initial credal set in the interval have the same asymptotic limit.
This is stronger than convergence of one selected inhomogeneous orbit.

Conversely, $\mathcal L_\star\neq\mathcal U_\star$ does not imply that a
particular orbit diverges, or even that two admissible orbits have different
limits. It says only that the envelope comparison is not sharp enough to
certify a common path independent limit. In that case
$[\mathcal L_\star,\mathcal U_\star]$ is the common asymptotic trap supplied
by the present theory.
\end{remark}

The next one dimensional example makes the distinction between the general trapping theorem and the coalescing case explicit.

\begin{example}[Arbitrary Switching among Contractions on a Binary State Space]\label{ex:binary}
Let $X=\{0,1\}$ and identify $\Prob(X)$ with $[0,1]$ through $p=P(\{1\})$. Fix $q\in(0,1)$ and numbers $0<r_-\le r_+<1$. Let
\[
 \mathcal A=\{q\},\qquad \mathcal B=[0,1],
\]
and, for $r\in[r_-,r_+]$, define
\[
 s_r(p)=q+r(p-q),
 \qquad
 F_r(\mathcal P)=s_r[\mathcal P].
\]
Every credal set in $[\mathcal A,\mathcal B]$ is an interval containing $q$. On the binary simplex, the metric \eqref{eq:rho} has the form $\rho(P_p,P_{p'})=c|p-p'|$ for a constant $c>0$ determined by the dense sequence $(u_j)$; hence $F_r$ is an $r$-contraction in the induced Hausdorff metric. For an interval containing $q$, larger $r$ also produces a larger image interval, so
\[
 F_\wedge=F_{r_-},\qquad F_\vee=F_{r_+}.
\]
The lower endpoint sequence is constantly $\{q\}$, while
\[
 F_{r_+}^n([0,1])=q+r_+^n([0,1]-q)\xrightarrow[d_H]{}\{q\}.
\]
Thus the envelopes coalesce and Corollary~\ref{cor:coalescence} applies. In fact, for any sequence $r_n\in[r_-,r_+]$,
\[
 \mathcal P_n
 =q+\left(\prod_{j=1}^nr_j\right)(\mathcal P_0-q)
 \xrightarrow[d_H]{}\{q\}.
\]
The point here is
that the envelope criterion certifies convergence uniformly over all admissible update sequences.
\end{example}

\begin{remark}[Interpreting the Envelope Gap]
When $\mathcal L_\star\neq\mathcal U_\star$,
Theorem~\ref{thm:cluster-envelope} still gives a conclusion for the entire
orbit,
\[
\inf_{\mathcal Q\in[\mathcal L_\star,\mathcal U_\star]}
d_H(\mathcal P_n,\mathcal Q)
\longrightarrow0
\]
uniformly over all admissible initial credal sets and switching sequences.
It does not assert that $\mathcal P_n$ itself converges. Different
subsequences may converge to different elements of the limiting interval.
The distance
$d_H(\mathcal L_\star,\mathcal U_\star)$ is therefore a natural diagnostic
of the width of the common asymptotic trap, but it is not claimed to be a
convergence rate.
\end{remark}

\section{Invariant Credal Sets for Imprecise Markov Semigroups}
\label{sec:imsg}

Sections~\ref{sec:order} and~\ref{sec:varying} study credal dynamics
through comparison under set inclusion. We now consider a different
stability mechanism in the setting of imprecise Markov semigroups
\citep{imsg}. Using the support functional representation of credal sets,
we investigate how uniform Wasserstein contraction of the admissible
Markov semigroups leads to a contractive evolution on the credal hyperspace.

Let $\Theta$ be a nonempty standard Borel space. An imprecise Markov
semigroup is a family
\[
\mathbf Q
=
\{\mathbf P^\theta:\theta\in\Theta\}
\]
of Markov semigroups on a common state space \citep{bakry}. For each
$\theta\in\Theta$,
$\mathbf P^\theta
=
(P_t^\theta)_{t\geq0}$
represents one possible transition mechanism for a continuous time Markov
process. The family $\mathbf Q$ represents ambiguity about the dynamics and,
when invariant laws exist, about the long run behavior of the process.

The geometry of the state space supplies the contraction mechanism.
Curvature bounds, gradient estimates, and contractive couplings can yield a
Wasserstein contraction uniformly over the model family
\citep{imsg}. We first use this geometry to construct a contractive semigroup
on the credal hyperspace. We then identify its invariant credal set with the
robust ergodic limit and characterize when that equilibrium is precise.
Finally, we compare two such evolutions and derive quantitative perturbation
bounds.

\subsection{The Induced Credal Semigroup}

Let $(E,d)$ be a compact metric space, and let
$\mathbf Q=\{\mathbf P^\theta:\theta\in\Theta\}$ be an IMSG. For
$f\in C(E)$, the space of real valued continuous functions on $E$, write
\[
 P_t^\theta f(x)
 \coloneqq \int_E f(y) p_t^\theta(x,\mathrm dy),
\]
where $p_t^\theta$ is the transition kernel associated with
$P_t^\theta$. For an initial probability measure $\mu\in\Prob(E)$, define its time $t$
distribution under model $\theta$ by
\[
 \mu P_t^\theta(A)
 \coloneqq 
 \int_E p_t^\theta(x,A) \mu(\mathrm dx),
 \qquad A\in\Borel(E).
\]

Write $\Lip(E)$ for the real valued Lipschitz functions on $E$, and let
$\Lip(f)$ denote the Lipschitz constant of $f \in \Lip(E)$. Since $E$ is compact, every
element of $\Lip(E)$ is bounded. Let $W_1$ be the $1$-Wasserstein metric on
$\Prob(E)$,
\begin{equation}\label{eq:wasserstein-duality}
 W_1(\mu,\nu)
 =
 \sup_{\Lip(f)\leq1}
 \left|\mu f-\nu f\right|.
\end{equation}
Assume first that there exists $\lambda>0$ such that
\begin{equation}\label{eq:uniform-wasserstein-contraction}
 W_1(\mu P_t^\theta,\nu P_t^\theta)
 \leq e^{-\lambda t}W_1(\mu,\nu)
\end{equation}
for every $\mu,\nu\in\Prob(E)$, $\theta\in\Theta$, and $t\geq0$.
Assume also that, for every $A\in\Borel(E)$, the map
\[
 (t,x,\theta)\longmapsto p_t^\theta(x,A)
\]
is Borel measurable on $[0,\infty)\times E\times\Theta$. 

The uniform Wasserstein estimate
\eqref{eq:uniform-wasserstein-contraction} is the main geometric hypothesis
of this section. It is stronger than ergodicity of each individual
semigroup: the same rate $\lambda$ must hold uniformly over the admissible
model family. Such an estimate can follow from uniform curvature bounds,
gradient estimates, or contractive coupling arguments
\citep{imsg,bakry}. Its role below is to transfer probability level
contraction to Hausdorff contraction on the credal hyperspace.

The joint measurability assumption is technical rather than contractive.
Together with the standard Borel structure of $\Theta$, it ensures that the
upper envelope over the admissible models is measurable, and that the
time consistent Nisio construction is well defined. We adopt these standing
conditions from the IMSG framework of \citet{imsg}.

For $f\in\Lip(E)$, define the one-step upper envelope
\[
 U_tf(x)\coloneqq \sup_{\theta\in\Theta}P_t^\theta f(x).
\]
For a finite partition
$\pi=(0=t_0<t_1<\cdots<t_n=t)$ of $[0,t]$, set
\[
 U_t^\pi f
 \coloneqq 
 U_{t_n-t_{n-1}}\circ\cdots\circ U_{t_1-t_0}f,
\]
and define
\begin{equation}\label{eq:nisio-semigroup}
 \mathcal T_tf
 \coloneqq 
 \sup_{\pi}U_t^\pi f.
\end{equation}
Under 
\citet[Assumptions~5.1 and 5.3]{imsg} and the Nisio construction of Section~5.2 thereof, this construction gives a
time consistent sublinear Markov semigroup. In particular,
\begin{align}
 \mathcal T_0&=\Id,
 &
 \mathcal T_{s+t}&=\mathcal T_s\circ\mathcal T_t,
 \label{eq:nisio-semigroup-property}\\
 \|\mathcal T_tf-\mathcal T_tg\|_\infty
 &\leq\|f-g\|_\infty,
 &
 \Lip(\mathcal T_tf)
 &\leq e^{-\lambda t}\Lip(f).
 \label{eq:nisio-contraction}
\end{align}
Each $\mathcal T_t$ is monotone, constant-preserving, positively
homogeneous, and subadditive.

Because $\Lip(E)$ is dense in $C(E)$ and $\mathcal T_t$ is nonexpansive in
the supremum norm, $\mathcal T_t$ extends uniquely to a nonexpansive map on
$C(E)$. We use the same notation for this extension. The semigroup property,
monotonicity, constant preservation, positive homogeneity, and subadditivity
pass to $C(E)$ by density and continuity.

Let $\Cred_E$ denote the hyperspace of nonempty weakly compact convex
subsets of $\Prob(E)$. For $\mathcal P\in\Cred_E$ and $t\geq0$, define
\begin{equation}\label{eq:dual-nisio-functional}
 H_{t,\mathcal P}(f)
 \coloneqq 
 h_{\mathcal P}(\mathcal T_tf),
 \qquad f\in C(E),
\end{equation}
where $h_{\mathcal P}(\cdot)$ denotes the support value defined in \eqref{supp-value}, and set
\begin{equation}\label{eq:dual-nisio-credal-map}
 \mathfrak N_t(\mathcal P)
 \coloneqq 
 \left\{
   \nu\in\Prob(E):
   \nu f\leq H_{t,\mathcal P}(f),
   \text{ for every }f\in C(E)
 \right\}.
\end{equation}

\begin{lemma}[Credal Representation of the Nisio Evolution]
\label{lem:nisio-credal-representation}
For every $t\geq0$ and $\mathcal P\in\Cred_E$,
$\mathfrak N_t(\mathcal P)$ is an element of $\Cred_E$ and
\begin{equation}\label{eq:nisio-support-identity}
 h_{\mathfrak N_t(\mathcal P)}(f)
 =
 h_{\mathcal P}(\mathcal T_tf)
 \qquad\text{for every }f\in C(E).
\end{equation}
\end{lemma}

\begin{proof}
The functional $H_{t,\mathcal P}$ is monotone, constant-preserving,
positively homogeneous, and subadditive. It is also continuous in the
supremum norm. Indeed,
\[
 \left|H_{t,\mathcal P}(f)-H_{t,\mathcal P}(g)\right|
 \leq
 \|\mathcal T_tf-\mathcal T_tg\|_\infty
 \leq\|f-g\|_\infty.
\]

Fix $f_0\in C(E)$. By the Hahn-Banach theorem for continuous sublinear
functionals, there exists a continuous linear functional
$\ell:C(E)\to\mathbb R$ such that
\[
 \ell(f)\leq H_{t,\mathcal P}(f)
 \quad\text{for every }f\in C(E),
 \qquad
 \ell(f_0)=H_{t,\mathcal P}(f_0).
\]
Since $H_{t,\mathcal P}(c)=c$ for every constant $c$, applying the
domination inequality to $1$ and $-1$ gives $\ell(1)=1$. If $g\geq0$, then
monotonicity gives $H_{t,\mathcal P}(-g)\leq H_{t,\mathcal P}(0)=0$.
Consequently,
\[
 -\ell(g)=\ell(-g)\leq H_{t,\mathcal P}(-g)\leq0,
\]
so $\ell(g)\geq0$. The Riesz-Markov representation theorem therefore gives a
probability measure $\nu_{f_0}\in\Prob(E)$ such that
$\ell(f)=\nu_{f_0}f$ for every $f\in C(E)$.

By construction, $\nu_{f_0}\in\mathfrak N_t(\mathcal P)$ and
\[
 \nu_{f_0}f_0=H_{t,\mathcal P}(f_0).
\]
Thus the supremum over $\mathfrak N_t(\mathcal P)$ attains
$H_{t,\mathcal P}(f_0)$. Since $f_0$ was arbitrary,
\eqref{eq:nisio-support-identity} follows.

The set $\mathfrak N_t(\mathcal P)$ is an intersection of weakly closed
half spaces in $\Prob(E)$, and hence is weakly closed and convex. It is
nonempty by the preceding argument. Since $E$ is compact,
$\Prob(E)$ is weakly compact, so $\mathfrak N_t(\mathcal P)$ is compact.
Therefore, $\mathfrak N_t(\mathcal P)\in\Cred_E$.
\end{proof}

Lemma~\ref{lem:nisio-credal-representation} connects the
function formulation of an IMSG with the credal hyperspace formulation of
this paper. The upper Nisio expectation is not merely assigned an abstract
sublinear limit. It is represented at every time by the support functional
of a compact credal set.
Let $d_{H,1}$ denote the Hausdorff metric on $\Cred_E$ induced by $W_1$.

\begin{lemma}[Hausdorff Wasserstein Support Formula]
\label{lem:wasserstein-support}
For $\mathcal P,\mathcal Q\in\Cred_E$,
\begin{equation}\label{eq:wasserstein-hausdorff-support}
 d_{H,1}(\mathcal P,\mathcal Q)
 =
 \sup_{\Lip(f)\leq1}
 \left|
 h_{\mathcal P}(f)-h_{\mathcal Q}(f)
 \right|.
\end{equation}
\end{lemma}

\begin{proof}
Let $\delta=d_{H,1}(\mathcal P,\mathcal Q)$. If $\Lip(f)\leq1$ and
$\mu\in\mathcal P$, choose, for any $\varepsilon>0$, a measure
$\nu\in\mathcal Q$ such that
$W_1(\mu,\nu)\leq\delta+\varepsilon$. By
\eqref{eq:wasserstein-duality},
\[
 \mu f
 \leq\nu f+\delta+\varepsilon
 \leq h_{\mathcal Q}(f)+\delta+\varepsilon.
\]
Taking the supremum over $\mu\in\mathcal P$, letting
$\varepsilon\downarrow0$, and reversing the roles of $\mathcal P$ and
$\mathcal Q$ gives
\[
 \sup_{\Lip(f)\leq1}
 \left|h_{\mathcal P}(f)-h_{\mathcal Q}(f)\right|
 \leq\delta.
\]

For the reverse inequality, fix $\mu_0\in\Prob(E)$ and translate
$\mathcal P$ and $\mathcal Q$ into the vector space
\[
 \mathcal M_0(E)
 \coloneqq 
 \{\sigma:\sigma\text{ is a finite signed measure on }E, \sigma(E)=0\}
\]
by considering
\[
 \mathcal P-\mu_0
 \qquad\text{and}\qquad
 \mathcal Q-\mu_0.
\]
Equip $\mathcal M_0(E)$ with the Kantorovich-Rubinstein norm. Translation
preserves the Hausdorff distance, so the standard support function formula
for closed convex subsets of a normed space gives
\[
 d_{H,1}(\mathcal P,\mathcal Q)
 =
 \sup_{\Lip(f)\leq1}
 \left|
   h_{\mathcal P}(f)-h_{\mathcal Q}(f)
 \right|.
\]
Indeed, Kantorovich-Rubinstein duality identifies the dual unit ball with
Lipschitz functions of Lipschitz constant at most one, modulo additive
constants, and the difference
$h_{\mathcal P}(f)-h_{\mathcal Q}(f)$ is unchanged when a constant is added
to $f$. This proves
\eqref{eq:wasserstein-hausdorff-support}.
\end{proof}

Lemma~\ref{lem:wasserstein-support} provides the metric bridge between the
function level Nisio dynamics and the induced evolution of credal sets. It
shows that Hausdorff distance in $W_1$ can be recovered entirely from
differences of support functionals evaluated on $1$-Lipschitz observables.
Consequently, the Lipschitz contraction of $\mathcal T_t$ transfers directly
to a Hausdorff contraction of $\mathfrak N_t$. This is the key step that
allows the fixed point theory on $\Cred_E$ to be applied in Theorem \ref{thm:imsg-fixed-credal-set} next.

\subsection{Invariant Equilibrium and Robust Ergodic Limit}

We now turn from the construction of the credal semigroup to its long term
behavior. The key point is that uniform Wasserstein contraction of the
underlying model family lifts to a strict Hausdorff-Wasserstein contraction
on credal sets, so the imprecise dynamics select a unique invariant credal
state that is independent of the initial credal set.

\begin{theorem}[Invariant Credal Set of an Imprecise Markov Semigroup]
\label{thm:imsg-fixed-credal-set}
Under the standing assumptions of this section, including
\eqref{eq:uniform-wasserstein-contraction}, the family
$(\mathfrak N_t)_{t\geq0}$ has the following properties.
\begin{enumerate}[label=(\alph*),leftmargin=*]
\item It is a semigroup of self-maps of $\Cred_E$,
\[
 \mathfrak N_0=\Id,
 \qquad
 \mathfrak N_{s+t}
 =
 \mathfrak N_t\circ\mathfrak N_s
 \quad\text{for all }s,t\geq0.
\]
\item For every $\mathcal P,\mathcal Q\in\Cred_E$ and $t\geq0$,
\begin{equation}\label{eq:nisio-hausdorff-contraction}
 d_{H,1}\bigl(
   \mathfrak N_t(\mathcal P),
   \mathfrak N_t(\mathcal Q)
 \bigr)
 \leq
 e^{-\lambda t}d_{H,1}(\mathcal P,\mathcal Q).
\end{equation}
\item There exists a unique credal set
$\mathcal K_\star\in\Cred_E$ satisfying
\begin{equation}\label{eq:common-invariant-credal-set}
 \mathfrak N_t(\mathcal K_\star)=\mathcal K_\star
 \qquad\text{for every }t\geq0.
\end{equation}
Moreover, for every $\mathcal P_0\in\Cred_E$ and $t\geq0$,
\begin{equation}\label{eq:imsg-credal-convergence}
 d_{H,1}\bigl(
   \mathfrak N_t(\mathcal P_0),
   \mathcal K_\star
 \bigr)
 \leq
 e^{-\lambda t}
 d_{H,1}(\mathcal P_0,\mathcal K_\star).
\end{equation}
\end{enumerate}
\end{theorem}

\begin{proof}
For $f\in C(E)$, Lemma~\ref{lem:nisio-credal-representation} and
\eqref{eq:nisio-semigroup-property} give
\begin{align*}
 h_{\mathfrak N_t(\mathfrak N_s(\mathcal P))}(f)
 =
 h_{\mathfrak N_s(\mathcal P)}(\mathcal T_tf)
 =
 h_{\mathcal P}(\mathcal T_s(\mathcal T_tf))
 =
 h_{\mathcal P}(\mathcal T_{s+t}f)
 =
 h_{\mathfrak N_{s+t}(\mathcal P)}(f).
\end{align*}
Compact convex sets are determined by their support functions, so
$\mathfrak N_{s+t}=\mathfrak N_t\circ\mathfrak N_s$. The identity at
$t=0$ follows in the same way from $\mathcal T_0=\Id$. This proves part
(a).

For part (b), Lemmas~\ref{lem:nisio-credal-representation} and \ref{lem:wasserstein-support} give
\begin{align*}
 d_{H,1}\bigl(
   \mathfrak N_t(\mathcal P),
   \mathfrak N_t(\mathcal Q)
 \bigr)=
 \sup_{\Lip(f)\leq1}
 \left|
 h_{\mathcal P}(\mathcal T_tf)
 -
 h_{\mathcal Q}(\mathcal T_tf)
 \right|.
\end{align*}
For every Lipschitz function $g$,
\[
 \left|h_{\mathcal P}(g)-h_{\mathcal Q}(g)\right|
 \leq
 \Lip(g)d_{H,1}(\mathcal P,\mathcal Q).
\]
Using \eqref{eq:nisio-contraction} with $g=\mathcal T_tf$ therefore gives
\[
 \left|
 h_{\mathcal P}(\mathcal T_tf)
 -
 h_{\mathcal Q}(\mathcal T_tf)
 \right|
 \leq
 e^{-\lambda t}d_{H,1}(\mathcal P,\mathcal Q),
\]
which proves \eqref{eq:nisio-hausdorff-contraction}.

Since $E$ is compact, $W_1$ metrizes weak convergence on $\Prob(E)$ and
$\Prob(E)$ is compact in $W_1$. The corresponding hyperspace of nonempty
compact convex sets is therefore complete in $d_{H,1}$.

Fix $t_0>0$. By \eqref{eq:nisio-hausdorff-contraction},
$\mathfrak N_{t_0}$ is a strict contraction with coefficient
$e^{-\lambda t_0}$. Banach's fixed point theorem gives a unique
$\mathcal K_\star\in\Cred_E$ such that
$\mathfrak N_{t_0}(\mathcal K_\star)=\mathcal K_\star$.

For $s\geq0$, the semigroup property gives
\[
 \mathfrak N_{t_0}\bigl(\mathfrak N_s(\mathcal K_\star)\bigr)
 =
 \mathfrak N_s\bigl(\mathfrak N_{t_0}(\mathcal K_\star)\bigr)
 =
 \mathfrak N_s(\mathcal K_\star).
\]
Thus $\mathfrak N_s(\mathcal K_\star)$ is also a fixed point of
$\mathfrak N_{t_0}$. Uniqueness implies
$\mathfrak N_s(\mathcal K_\star)=\mathcal K_\star$. Since $s$ was
arbitrary, \eqref{eq:common-invariant-credal-set} follows.

Finally, invariance and \eqref{eq:nisio-hausdorff-contraction} give
\[
 d_{H,1}\bigl(
   \mathfrak N_t(\mathcal P_0),
   \mathcal K_\star
 \bigr)
 =
 d_{H,1}\bigl(
   \mathfrak N_t(\mathcal P_0),
   \mathfrak N_t(\mathcal K_\star)
 \bigr)
 \leq
 e^{-\lambda t}
 d_{H,1}(\mathcal P_0,\mathcal K_\star).
\]
\end{proof}

Theorem~\ref{thm:imsg-fixed-credal-set} gives a continuous time counterpart
to the contraction results of Section~\ref{sec:classical}. Its fixed point
is not an invariant law for one selected Markov model. It is a single
invariant credal set for the time consistent dynamics generated by the
whole imprecise family. The geometric contraction assumption therefore
removes both dependence on the initial credal set and ambiguity about which
fixed credal state is selected.

The result complements the uniform trapping conclusion of
Theorem~\ref{thm:cluster-envelope}. There, changing updates are controlled by
order bounds and every orbit approaches a limiting credal interval, while
convergence to a single credal set requires envelope coalescence. Here the
geometry of the state space supplies a strict Hausdorff contraction, so
uniqueness and convergence follow directly.

\begin{corollary}[Credal Representation of the Robust Ergodic Limit]
\label{cor:imsg-ergodic-functional}
Let
 $D\coloneqq \sup_{x,y\in E}d(x,y)$
be the diameter of $E$, and define
 $R:C(E)\to\mathbb R$, 
 $R(f)\coloneqq h_{\mathcal K_\star}(f)$.
Then $R$ is monotone, constant-preserving, positively homogeneous, and
subadditive. It is invariant under the Nisio semigroup,
\begin{equation}\label{eq:imsg-functional-invariance}
 R(\mathcal T_tf)=R(f)
 \qquad
 \text{for every }f\in C(E)\text{ and }t\geq0.
\end{equation}
For every $f\in\Lip(E)$,
\begin{equation}\label{eq:imsg-functional-convergence}
 \sup_{x\in E}
 \left|
   \mathcal T_tf(x)-R(f)
 \right|
 \leq
 D e^{-\lambda t}\Lip(f).
\end{equation}
Moreover, for every $f\in C(E)$,
\begin{equation}\label{eq:imsg-continuous-functional-convergence}
 \|\mathcal T_tf-R(f)\|_\infty
 \longrightarrow 0
 \qquad\text{as }t\to\infty,
\end{equation}
where $R(f)$ is identified with the corresponding constant function.
\end{corollary}

\begin{proof}
The structural properties of $R$ follow because it is the support functional
of the credal set $\mathcal K_\star$. In addition, for
$f,g\in C(E)$,
\begin{equation}\label{eq:R-supnorm-continuity}
 |R(f)-R(g)|
 \leq
 \|f-g\|_\infty.
\end{equation}

For invariance, Lemma~\ref{lem:nisio-credal-representation} and
\eqref{eq:common-invariant-credal-set} give
\[
 R(\mathcal T_tf)
 =
 h_{\mathcal K_\star}(\mathcal T_tf)
 =
 h_{\mathfrak N_t(\mathcal K_\star)}(f)
 =
 h_{\mathcal K_\star}(f)
 =
 R(f)
\]
for every $f\in C(E)$.
Now let $f\in\Lip(E)$. For $x\in E$,
Lemma~\ref{lem:nisio-credal-representation} gives
\[
 \mathcal T_tf(x)
 =
 h_{\{\delta_x\}}(\mathcal T_tf)
 =
 h_{\mathfrak N_t(\{\delta_x\})}(f).
\]
Therefore,
\begin{align*}
 \left|
   \mathcal T_tf(x)-R(f)
 \right|
 &\leq
 \Lip(f)
 d_{H,1}\bigl(
   \mathfrak N_t(\{\delta_x\}),
   \mathcal K_\star
 \bigr)\\
 &\leq
 e^{-\lambda t}\Lip(f)
 d_{H,1}\bigl(
   \{\delta_x\},
   \mathcal K_\star
 \bigr).
\end{align*}
Since
\[
 W_1(\delta_x,\mu)\leq D
 \qquad
 \text{for every }\mu\in\Prob(E),
\]
we have
\[
 d_{H,1}\bigl(
   \{\delta_x\},
   \mathcal K_\star
 \bigr)
 \leq D.
\]
Taking the supremum over $x\in E$ proves
\eqref{eq:imsg-functional-convergence}.

Finally, fix $f\in C(E)$ and $\varepsilon>0$. Since $\Lip(E)$ is dense in
$C(E)$, choose $g\in\Lip(E)$ such that
\[
 \|f-g\|_\infty<\varepsilon.
\]
Using supremum norm nonexpansiveness of $\mathcal T_t$,
\eqref{eq:R-supnorm-continuity}, and
\eqref{eq:imsg-functional-convergence}, we obtain
\begin{align*}
 \|\mathcal T_tf-R(f)\|_\infty
 &\leq
 \|\mathcal T_tf-\mathcal T_tg\|_\infty
 +
 \|\mathcal T_tg-R(g)\|_\infty
 +
 |R(g)-R(f)|\\
 &\leq
 2\varepsilon
 +
 D e^{-\lambda t}\Lip(g).
\end{align*}
Taking the limit superior as $t\to\infty$ gives
\[
 \limsup_{t\to\infty}
 \|\mathcal T_tf-R(f)\|_\infty
 \leq 2\varepsilon.
\]
Since $\varepsilon>0$ was arbitrary,
\eqref{eq:imsg-continuous-functional-convergence} follows.
\end{proof}

Corollary~\ref{cor:imsg-ergodic-functional} gives the function level robust
ergodic limit a geometric interpretation. The limiting sublinear
expectation is exactly the upper expectation generated by the invariant
credal set $\mathcal K_\star$. The exponential estimate is available for
Lipschitz observables, while density and nonexpansiveness extend convergence
to every continuous observable. Thus the asymptotic object can be viewed
either as a nonlinear functional on $C(E)$ or as a fixed point in the credal
hyperspace.

\begin{remark}[Static and Time Consistent Evolutions]
For a fixed horizon $t$, one may instead define the static hull
\[
 \mathfrak M_t(\mathcal P)
 \coloneqq 
 \clco\left(
   \bigcup_{\theta\in\Theta}
   \{\mu P_t^\theta:\mu\in\mathcal P\}
 \right).
\]
The same support function argument, together with
$\Lip(P_t^\theta f)\leq e^{-\lambda t}\Lip(f)$ uniformly in $\theta$,
shows that
 $d_{H,1}\bigl(\mathfrak M_t(\mathcal P),\mathfrak M_t(\mathcal Q)\bigr)
 \leq e^{-\lambda t}d_{H,1}(\mathcal P,\mathcal Q)$.
In general, however,
$(\mathfrak M_t)_{t\geq0}$ is not a semigroup. The composition
$\mathfrak M_s\circ\mathfrak M_t$ permits a change of model between the two
time intervals, whereas $\mathfrak M_{s+t}$ uses one model over the entire
interval.

The Nisio operator resolves this lack of time consistency by allowing
switching along partitions and imposing the dynamic programming principle.
Moreover,
 $\mathfrak M_t(\mathcal P)
 \subseteq
 \mathfrak N_t(\mathcal P)$.
Indeed, for every $f\in\Lip(E)$,
\[
\begin{aligned}
 h_{\mathfrak M_t(\mathcal P)}(f)
 &=
 \sup_{\mu\in\mathcal P, \theta\in\Theta}
 \mu P_t^\theta f\\
 &\leq
 \sup_{\mu\in\mathcal P}
 \mu\left(\sup_{\theta\in\Theta}P_t^\theta f\right)\\
 &=
 h_{\mathcal P}(U_tf)
 \leq
 h_{\mathcal P}(\mathcal T_tf)
 =
 h_{\mathfrak N_t(\mathcal P)}(f).
\end{aligned}
\]
Since $\Lip(E)$ is dense in $C(E)$ and support functionals are continuous
in the supremum norm, the inequality extends to every $f\in C(E)$.
Therefore,
\[
 \mathfrak M_t(\mathcal P)\subseteq\mathfrak N_t(\mathcal P).
\]
\end{remark}

\begin{example}[A Uniformly Contractive Two State IMSG with a Non-Singleton Equilibrium]
\label{ex:two-state-imsg}
Let $E=\{0,1\}$ with $d(0,1)=1$, and write $P_p$ for the probability
measure satisfying $P_p(\{1\})=p$. Take $\Theta=\{1,2\}$ and let the two
Markov semigroups have generators
\[
 Q^\theta
 =
 \begin{pmatrix}
  -a_\theta & a_\theta\\
  b_\theta & -b_\theta
 \end{pmatrix},
 \qquad a_\theta,b_\theta>0.
\]
Put
\[
 c_\theta\coloneqq a_\theta+b_\theta,
 \qquad
 \pi_\theta\coloneqq \frac{a_\theta}{a_\theta+b_\theta}.
\]
A direct computation gives
\[
 P_p P_t^\theta
 =
 P_{\pi_\theta+e^{-c_\theta t}(p-\pi_\theta)}.
\]
Since $W_1(P_p,P_q)=|p-q|$ on this state space,
\[
 W_1(P_p P_t^\theta,P_q P_t^\theta)
 =e^{-c_\theta t}W_1(P_p,P_q).
\]
Thus the standing contraction assumption holds with
$\lambda=\min_{\theta\in\Theta}c_\theta$; the measurability assumptions are
automatic because $\Theta$ is finite.

For a concrete choice, take
\[
 (a_1,b_1)=(1,3),
 \qquad
 (a_2,b_2)=(3,1).
\]
Then $c_1=c_2=4$, so the induced credal semigroup is an
$e^{-4t}$-contraction, while the two stationary laws are $P_{1/4}$ and
$P_{3/4}$. Let $\mathcal K_\star$ be the unique common invariant credal set
from Theorem~\ref{thm:imsg-fixed-credal-set}. Fix
$\mu\in\mathcal K_\star$. By the preceding remark,
\[
 \mu P_t^\theta
 \in
 \mathfrak M_t(\mathcal K_\star)
 \subseteq
 \mathfrak N_t(\mathcal K_\star)
 =\mathcal K_\star
 \qquad(t\geq0).
\]
Letting $t\to\infty$ and using closedness of $\mathcal K_\star$ gives
$P_{\pi_\theta}\in\mathcal K_\star$ for both $\theta=1,2$. Convexity
therefore yields
\[
 \{P_p:p\in[1/4,3/4]\}
 =\co\{P_{1/4},P_{3/4}\}
 \subseteq\mathcal K_\star.
\]
In particular, $\mathcal K_\star$ is not a singleton. The example shows that
strict contraction of the credal dynamics removes dependence on the initial
credal set while allowing persistent ambiguity about the long run law.
\end{example}

The example motivates an exact criterion for when the robust equilibrium is
precise rather than genuinely credal.

\begin{corollary}[When the Robust Equilibrium is Precise]
\label{cor:imsg-singleton}
The invariant credal set $\mathcal K_\star$ is a singleton if and only if
the functional $R:C(E)\to\mathbb R$ in
Corollary~\ref{cor:imsg-ergodic-functional} is linear.
\end{corollary}

\begin{proof}
If $\mathcal K_\star=\{\mu_\star\}$, then
$R(f)=\mu_\star f$, so $R$ is linear.

Conversely, suppose $R=h_{\mathcal K_\star}$ is linear. Then
\[
 \sup_{\mu\in\mathcal K_\star}\mu f
 =
 R(f)
 =
 -R(-f)
 =
 \inf_{\mu\in\mathcal K_\star}\mu f
\]
for every $f\in C(E)$. Hence all measures in $\mathcal K_\star$ have the
same integral against every continuous function. Since $C(E)$ separates
probability measures, $\mathcal K_\star$ contains only one measure.
\end{proof}

This criterion separates two forms of asymptotic stability.
Theorem~\ref{thm:imsg-fixed-credal-set} always removes dependence on the
initial
credal set. Corollary~\ref{cor:imsg-singleton}, on the other hand, identifies the additional
condition under which ambiguity about the long run law disappears as well.
A non-singleton $\mathcal K_\star$ therefore represents persistent model
imprecision, not failure of convergence.

\subsection{Perturbation Stability}
The contractive fixed point construction is also stable under perturbations.
In the present setting, the support representation identifies the discrepancy
of the induced credal maps exactly with that of the corresponding Nisio
semigroups, while one-step kernel estimates yield an explicit equilibrium
bound.

Consider a second imprecise Markov semigroup on the same compact metric space
$(E,d)$ satisfying the standing assumptions of this section, including the
uniform Wasserstein contraction hypothesis with rate
$\widetilde\lambda>0$. Its parameter space
$\widetilde\Theta$ may differ from $\Theta$. Write
$\widetilde P_t^{\widetilde\theta}$ and
$\widetilde p_t^{\widetilde\theta}$ for its Markov operators and transition
kernels, and denote its Nisio semigroup, induced credal semigroup, and common
invariant credal set by
\[
 (\widetilde{\mathcal T}_t)_{t\geq0},
 \qquad
 (\widetilde{\mathfrak N}_t)_{t\geq0},
 \qquad
 \widetilde{\mathcal K}_\star,
\]
respectively.
For $t>0$, define
\[
 \varepsilon_t
 \coloneqq 
 \sup_{\substack{f\in\Lip(E)\\ \Lip(f)\leq1}}
 \left\|
   \mathcal T_t f-\widetilde{\mathcal T}_t f
 \right\|_\infty.
\]

\begin{proposition}[Perturbation Stability of the Invariant Credal Set]
\label{prop:imsg-perturbation}
We have that $\varepsilon_t<\infty$, and
\begin{equation}
\label{eq:imsg-map-perturbation-discrepancy}
\varepsilon_t
=
\sup_{x\in E}
d_{H,1}\bigl(
  \mathfrak N_t(\{\delta_x\}),
  \widetilde{\mathfrak N}_t(\{\delta_x\})
\bigr)
=
\sup_{\mathcal P\in\Cred_E}
d_{H,1}\bigl(
  \mathfrak N_t(\mathcal P),
  \widetilde{\mathfrak N}_t(\mathcal P)
\bigr).
\end{equation}
Moreover, for every $t>0$,
\begin{equation}
\label{eq:imsg-fixed-point-perturbation}
d_{H,1}\bigl(
  \mathcal K_\star,
  \widetilde{\mathcal K}_\star
\bigr)
\leq
\frac{\varepsilon_t}
     {1-e^{-\max\{\lambda,\widetilde\lambda\}t}}.
\end{equation}
Consequently,
\[
\begin{aligned}
\sup_{\substack{f\in\Lip(E)\\ \Lip(f)\leq1}}
\left|
  h_{\mathcal K_\star}(f)
  -
  h_{\widetilde{\mathcal K}_\star}(f)
\right|
=
d_{H,1}\bigl(
  \mathcal K_\star,
  \widetilde{\mathcal K}_\star
\bigr)
\leq
\frac{\varepsilon_t}
     {1-e^{-\max\{\lambda,\widetilde\lambda\}t}}.
\end{aligned}
\]
\end{proposition}

\begin{proof}
Both Nisio semigroups are translation invariant. Hence, for any fixed
$x_0\in E$,
\[
 \mathcal T_t\bigl(f-f(x_0)\bigr)
 -
 \widetilde{\mathcal T}_t\bigl(f-f(x_0)\bigr)
 =
 \mathcal T_t f-\widetilde{\mathcal T}_t f.
\]
The supremum defining $\varepsilon_t$ may therefore be restricted to
functions satisfying $f(x_0)=0$. Since $E$ is compact, this normalized class
is uniformly bounded, and thus $\varepsilon_t<\infty$.

Fix $\mathcal P\in\Cred_E$. Lemmas~\ref{lem:nisio-credal-representation}
and~\ref{lem:wasserstein-support} give
\[
\begin{aligned}
&d_{H,1}\bigl(
  \mathfrak N_t(\mathcal P),
  \widetilde{\mathfrak N}_t(\mathcal P)
\bigr)
=
\sup_{\substack{f\in\Lip(E)\\ \Lip(f)\leq1}}
\left|
  h_{\mathcal P}(\mathcal T_t f)
  -
  h_{\mathcal P}(\widetilde{\mathcal T}_t f)
\right|.
\end{aligned}
\]
For any bounded functions $g$ and $h$,
\[
 \left|
   h_{\mathcal P}(g)-h_{\mathcal P}(h)
 \right|
 \leq
 \|g-h\|_\infty.
\]
It follows that
\[
d_{H,1}\bigl(
  \mathfrak N_t(\mathcal P),
  \widetilde{\mathfrak N}_t(\mathcal P)
\bigr)
\leq
\varepsilon_t.
\]
Taking the supremum over $\mathcal P\in\Cred_E$ gives one direction in
\eqref{eq:imsg-map-perturbation-discrepancy}.

Conversely, for every $x\in E$, the same two lemmas give
\[
\begin{aligned}
d_{H,1}\bigl(
  \mathfrak N_t(\{\delta_x\}),
  \widetilde{\mathfrak N}_t(\{\delta_x\})
\bigr)
&=
\sup_{\substack{f\in\Lip(E)\\ \Lip(f)\leq1}}
\left|
  h_{\{\delta_x\}}(\mathcal T_t f)
  -
  h_{\{\delta_x\}}(\widetilde{\mathcal T}_t f)
\right|
\\
&=
\sup_{\substack{f\in\Lip(E)\\ \Lip(f)\leq1}}
\left|
  \mathcal T_t f(x)
  -
  \widetilde{\mathcal T}_t f(x)
\right|.
\end{aligned}
\]
Therefore,
\[
\begin{aligned}
\sup_{x\in E}
d_{H,1}\bigl(
  \mathfrak N_t(\{\delta_x\}),
  \widetilde{\mathfrak N}_t(\{\delta_x\})
\bigr)
&=
\sup_{x\in E}
\sup_{\substack{f\in\Lip(E)\\ \Lip(f)\leq1}}
\left|
  \mathcal T_t f(x)
  -
  \widetilde{\mathcal T}_t f(x)
\right|
\\
&=
\sup_{\substack{f\in\Lip(E)\\ \Lip(f)\leq1}}
\left\|
  \mathcal T_t f-\widetilde{\mathcal T}_t f
\right\|_\infty
=
\varepsilon_t.
\end{aligned}
\]
Since singleton credal sets belong to $\Cred_E$, this proves
\eqref{eq:imsg-map-perturbation-discrepancy}.
Set
 $r
 \coloneqq 
 d_{H,1}(
   \mathcal K_\star,
   \widetilde{\mathcal K}_\star
 )$.
Using the invariance of the two credal sets, the triangle inequality,
\eqref{eq:nisio-hausdorff-contraction}, and
\eqref{eq:imsg-map-perturbation-discrepancy}, we obtain
\[
\begin{aligned}
r
=
d_{H,1}\bigl(
  \mathfrak N_t(\mathcal K_\star),
  \widetilde{\mathfrak N}_t(
    \widetilde{\mathcal K}_\star
  )
\bigr)
\leq
d_{H,1}\bigl(
  \mathfrak N_t(\mathcal K_\star),
  \mathfrak N_t(
    \widetilde{\mathcal K}_\star
  )
\bigr)
+
d_{H,1}\bigl(
  \mathfrak N_t(
    \widetilde{\mathcal K}_\star
  ),
  \widetilde{\mathfrak N}_t(
    \widetilde{\mathcal K}_\star
  )
\bigr)
\leq
e^{-\lambda t}r+\varepsilon_t.
\end{aligned}
\]
Thus
\[
 r
 \leq
 \frac{\varepsilon_t}{1-e^{-\lambda t}}.
\]
Interchanging the roles of the two semigroups similarly gives
\[
\begin{aligned}
r
\leq
d_{H,1}\bigl(
  \mathfrak N_t(\mathcal K_\star),
  \widetilde{\mathfrak N}_t(\mathcal K_\star)
\bigr)
+
d_{H,1}\bigl(
  \widetilde{\mathfrak N}_t(\mathcal K_\star),
  \widetilde{\mathfrak N}_t(
    \widetilde{\mathcal K}_\star
  )
\bigr)
\leq
\varepsilon_t+e^{-\widetilde\lambda t}r,
\end{aligned}
\]
and hence
\[
 r
 \leq
 \frac{\varepsilon_t}
      {1-e^{-\widetilde\lambda t}}.
\]
Taking the better of these two estimates yields
\[
 r
 \leq
 \min\left\{
   \frac{\varepsilon_t}{1-e^{-\lambda t}},
   \frac{\varepsilon_t}{1-e^{-\widetilde\lambda t}}
 \right\}
 =
 \frac{\varepsilon_t}
      {1-e^{-\max\{\lambda,\widetilde\lambda\}t}},
\]
which proves \eqref{eq:imsg-fixed-point-perturbation}.

Finally, Lemma~\ref{lem:wasserstein-support} gives
\[
d_{H,1}\bigl(
  \mathcal K_\star,
  \widetilde{\mathcal K}_\star
\bigr)
=
\sup_{\substack{f\in\Lip(E)\\ \Lip(f)\leq1}}
\left|
  h_{\mathcal K_\star}(f)
  -
  h_{\widetilde{\mathcal K}_\star}(f)
\right|,
\]
which proves the remaining assertion.
\end{proof}

The finite time discrepancy in
Proposition~\ref{prop:imsg-perturbation} can be controlled directly from the
one-step upper envelopes or, equivalently, from the associated convex sets of
transition kernels.

Let $U_s$ and $\widetilde U_s$ denote the one-step upper envelopes of the two
imprecise Markov semigroups. Suppose that there exist $c\geq0$ and $s_0>0$
such that
\begin{equation}
\label{eq:one-step-imsg-perturbation}
\left\|
  U_s f-\widetilde U_s f
\right\|_\infty
\leq
cs\Lip(f),
\qquad
  0<s\leq s_0, f\in\Lip(E).
\end{equation}

\begin{corollary}[One-Step Kernel Perturbation Bound]
\label{cor:imsg-one-step-perturbation}
We have that
\[
 \varepsilon_t
 \leq
 ct
 \qquad
 (0<t\leq s_0),
\]
and consequently
\begin{equation}
\label{eq:imsg-kernel-perturbation-bound}
d_{H,1}\bigl(
  \mathcal K_\star,
  \widetilde{\mathcal K}_\star
\bigr)
\leq
\frac{c}{\max\{\lambda,\widetilde\lambda\}}.
\end{equation}

For $s>0$ and $x\in E$, define
\[
 \Gamma_s(x)
 \coloneqq 
 \clco\bigl(
   \{p_s^\theta(x,\cdot):\theta\in\Theta\}
 \bigr), \qquad 
 \widetilde\Gamma_s(x)
 \coloneqq 
 \clco\bigl(
   \{
     \widetilde p_s^{\widetilde\theta}(x,\cdot):
     \widetilde\theta\in\widetilde\Theta
   \}
 \bigr).
\]
Condition~\eqref{eq:one-step-imsg-perturbation} is equivalent to the
parameterization-independent estimate
\begin{equation}
\label{eq:one-step-kernel-set-perturbation}
\sup_{x\in E}
d_{H,1}\bigl(
  \Gamma_s(x),
  \widetilde\Gamma_s(x)
\bigr)
\leq
cs,
\qquad
0<s\leq s_0.
\end{equation}
In particular, if the two model families are indexed by the same parameter
space $\Theta$, then
\eqref{eq:one-step-imsg-perturbation} holds whenever
\begin{equation}
\label{eq:markov-kernel-family-perturbation}
\sup_{\theta\in\Theta}
\sup_{x\in E}
W_1\bigl(
  p_s^\theta(x,\cdot),
  \widetilde p_s^\theta(x,\cdot)
\bigr)
\leq
cs,
\qquad
0<s\leq s_0.
\end{equation}
\end{corollary}

\begin{proof}
Each one step upper envelope is monotone and translation invariant, and
hence is nonexpansive in the supremum norm. The two uniform Wasserstein contraction
assumptions also give
\[
 \Lip(U_s g)
 \leq
 e^{-\lambda s}\Lip(g)
 \leq
 \Lip(g)
\qquad \text{and} \qquad 
 \Lip(\widetilde U_s g)
 \leq
 e^{-\widetilde\lambda s}\Lip(g)
 \leq
 \Lip(g),
\]
for every $g\in\Lip(E)$.

Fix $0<t\leq s_0$ and a finite partition $\pi$ of $[0,t]$. Let
$s_1,\ldots,s_n$ be its increments, listed in the order in which the
corresponding one-step factors are applied. Thus
\[
 s_i>0,
 \qquad
 \sum_{i=1}^n s_i=t,
\]
and the associated partition operators can be written as
\[
 U_t^\pi
 =
 U_{s_n}\circ\cdots\circ U_{s_1},
 \qquad
 \widetilde U_t^\pi
 =
 \widetilde U_{s_n}\circ\cdots\circ
 \widetilde U_{s_1}.
\]
For $i=0,\ldots,n$, define the hybrid compositions
\[
 H_i
 \coloneqq 
 U_{s_n}\circ\cdots\circ U_{s_{i+1}}
 \circ
 \widetilde U_{s_i}\circ\cdots\circ
 \widetilde U_{s_1}f,
\]
where an empty composition is understood as the identity. Then
\[
 H_0=U_t^\pi f,
 \qquad
 H_n=\widetilde U_t^\pi f.
\]
For $i=1,\ldots,n$, put
\[
 g_{i-1}
 \coloneqq 
 \widetilde U_{s_{i-1}}\circ\cdots\circ
 \widetilde U_{s_1}f,
\]
with $g_0=f$. Supremum-norm nonexpansiveness of the common outer composition
gives
\[
\begin{aligned}
\|H_i-H_{i-1}\|_\infty
\leq
\left\|
  \widetilde U_{s_i}g_{i-1}
  -
  U_{s_i}g_{i-1}
\right\|_\infty
\leq
cs_i\Lip(g_{i-1})
\leq
cs_i\Lip(f),
\end{aligned}
\]
where the second inequality follows from
\eqref{eq:one-step-imsg-perturbation}, and the last one follows because the
operators $\widetilde U_s$ do not increase Lipschitz constants. Therefore,
\[
\begin{aligned}
\left\|
  U_t^\pi f-\widetilde U_t^\pi f
\right\|_\infty
\leq
\sum_{i=1}^n
\|H_i-H_{i-1}\|_\infty
\leq
c\sum_{i=1}^n s_i\Lip(f)
=
ct\Lip(f).
\end{aligned}
\]

Writing $\Pi_t$ for the collection of finite partitions of $[0,t]$, the
Nisio construction gives
\[
 \mathcal T_t f
 =
 \sup_{\pi\in\Pi_t}U_t^\pi f,
 \qquad
 \widetilde{\mathcal T}_t f
 =
 \sup_{\pi\in\Pi_t}\widetilde U_t^\pi f,
\]
where both suprema are taken pointwise. Since
\[
 \left|
   \sup_{\pi\in\Pi_t}a_\pi
   -
   \sup_{\pi\in\Pi_t}b_\pi
 \right|
 \leq
 \sup_{\pi\in\Pi_t}|a_\pi-b_\pi|,
\]
we obtain
\[
 \left\|
   \mathcal T_t f-\widetilde{\mathcal T}_t f
 \right\|_\infty
 \leq
 ct\Lip(f).
\]
Consequently,
\[
 \varepsilon_t\leq ct,
 \qquad
0<t\leq s_0.
\]

Set
 $\alpha
 \coloneqq 
 \max\{\lambda,\widetilde\lambda\}$.
Proposition~\ref{prop:imsg-perturbation} now gives
\[
d_{H,1}\bigl(
  \mathcal K_\star,
  \widetilde{\mathcal K}_\star
\bigr)
\leq
\frac{ct}{1-e^{-\alpha t}},
\qquad
0<t\leq s_0.
\]
Letting $t\downarrow0$ and using
 $\lim_{t\downarrow0}
 \frac{t}{1-e^{-\alpha t}}
 =
 \frac{1}{\alpha}$
proves \eqref{eq:imsg-kernel-perturbation-bound}.

It remains to establish the kernel formulations. Closed convexification does
not change support values, so, for every $s>0$, $x\in E$, and
$f\in\Lip(E)$,
\[
\begin{aligned}
U_s f(x)
=
\sup_{\theta\in\Theta}
\int_E f(y) p_s^\theta(x,\mathrm dy)
=
h_{\Gamma_s(x)}(f).
\end{aligned}
\]
Similarly,
 $\widetilde U_s f(x)
 =
 h_{\widetilde\Gamma_s(x)}(f)$.
Lemma~\ref{lem:wasserstein-support} therefore gives
\[
\begin{aligned}
d_{H,1}\bigl(
  \Gamma_s(x),
  \widetilde\Gamma_s(x)
\bigr)
&=
\sup_{\substack{f\in\Lip(E)\\ \Lip(f)\leq1}}
\left|
  U_s f(x)-\widetilde U_s f(x)
\right|.
\end{aligned}
\]
Taking the supremum over $x\in E$ yields
\[
\begin{aligned}
\sup_{x\in E}
d_{H,1}\bigl(
  \Gamma_s(x),
  \widetilde\Gamma_s(x)
\bigr)
&=
\sup_{\substack{f\in\Lip(E)\\ \Lip(f)\leq1}}
\left\|
  U_s f-\widetilde U_s f
\right\|_\infty.
\end{aligned}
\]
By positive homogeneity, this proves the equivalence of
\eqref{eq:one-step-imsg-perturbation} and
\eqref{eq:one-step-kernel-set-perturbation}.

Finally, suppose that the two families have the common parameter space
$\Theta$ and satisfy
\eqref{eq:markov-kernel-family-perturbation}. Kantorovich--Rubinstein
duality gives, for every $\theta\in\Theta$, $x\in E$, and
$f\in\Lip(E)$,
\[
\begin{aligned}
\left|
  P_s^\theta f(x)-\widetilde P_s^\theta f(x)
\right|
\leq
\Lip(f)
W_1\bigl(
  p_s^\theta(x,\cdot),
  \widetilde p_s^\theta(x,\cdot)
\bigr)
\leq
cs\Lip(f).
\end{aligned}
\]
Using the elementary inequality
\[
 \left|
   \sup_{\theta\in\Theta}a_\theta
   -
   \sup_{\theta\in\Theta}b_\theta
 \right|
 \leq
 \sup_{\theta\in\Theta}|a_\theta-b_\theta|,
\]
we conclude that
\[
\begin{aligned}
\left|
  U_s f(x)-\widetilde U_s f(x)
\right|
\leq
\sup_{\theta\in\Theta}
\left|
  P_s^\theta f(x)-\widetilde P_s^\theta f(x)
\right|
\leq
cs\Lip(f).
\end{aligned}
\]
Taking the supremum over $x\in E$ proves
\eqref{eq:one-step-imsg-perturbation}.
\end{proof}

Proposition~\ref{prop:imsg-perturbation} and
Corollary~\ref{cor:imsg-one-step-perturbation} show that the robust
equilibrium is stable at two complementary levels. The first estimate is
intrinsic to the finite time Nisio evolutions, while the second reduces the
comparison to parameterization independent one step kernel sets.

\section{Conclusion}\label{sec:conclusion}

This paper develops fixed point and comparison methods for dynamics whose
state variable is a credal set. The inclusion order gives the first main
mechanism. Credal intervals are complete lattices, and Hausdorff continuity
turns the extremal fixed points supplied by Knaster-Tarski into limits of the
ordinary lower and upper endpoint iterations. For a family of changing
updates, the associated envelope maps provide finite time bounds for every
switching path and force the Hausdorff distance of each path from a common
limiting credal interval to tend to zero uniformly over the update sequence
and the initial credal set. Consequently, all Hausdorff cluster points lie
in that interval. Coalescence of its endpoints yields a common asymptotic
state uniformly over all admissible trajectories.

The imprecise Markov semigroup application supplies a stronger metric
mechanism. Uniform Wasserstein contraction of the admissible Markov
semigroups induces Hausdorff-Wasserstein contraction of the time consistent
credal evolution. The resulting semigroup has a unique credal set invariant
at every time, attracts every initial credal set exponentially, and
represents the robust ergodic limit through its support functional. The same
support representation yields quantitative perturbation stability: the
distance between the invariant credal sets of two such evolutions is
controlled by the finite time discrepancy of their Nisio semigroups, and a
parameterization independent one step kernel set estimate gives an explicit
equilibrium bound. Example~\ref{ex:two-state-imsg} shows that an invariant
credal set can nevertheless be non-singleton: contractive stability removes
dependence on the initial condition without necessarily eliminating
persistent ambiguity about the long run law.

The compactness, completeness, and support coordinate results identify the
hyperspace structure behind these conclusions and also place
Hausdorff-continuous or contractive credal updates within the scope of
classical fixed point theory. The Bayesian tilting subsection marks the
boundary of what that structure alone provides: Hausdorff continuity yields
existence, but a common positive tilt is projectively isometric and supplies
no convergence mechanism.

Several directions remain open. Extending the semigroup analysis to noncompact Polish spaces would require
suitable moment, tightness, and invariant domain conditions. It would also be useful to
study asymptotic, eventual, or state-dependent contractions,
differentiability and sharper local sensitivity properties of the invariant
credal set, and stochastic or adaptive switching rules that may yield
sharper conclusions than deterministic envelope trapping.

\section*{Acknowledgments} 
Michele Caprio gratefully acknowledges support from the Prob\_AI Hub and the London Mathematical Society. Mengqi Chen is supported by the Warwick Statistics Centre for Doctoral Training. OpenAI's ChatGPT Astra was used to assist with proofreading, and with improving the clarity and presentation of the manuscript.

\bibliographystyle{plainnat}
\bibliography{references}

\end{document}